\documentclass{article}
\usepackage[preprint]{neurips_2026}
\usepackage[utf8]{inputenc}
\usepackage[T1]{fontenc}
\usepackage{hyperref}
\usepackage{url}
\usepackage{booktabs}
\usepackage{amsfonts}
\usepackage{amsmath}
\usepackage{amssymb}
\usepackage{amsthm}
\usepackage{nicefrac}
\usepackage{microtype}
\usepackage{xcolor}
\usepackage{natbib}
\usepackage{tikz}
\usepackage{graphicx}
\usepackage[table]{xcolor}
\usepackage{multirow}
\usepackage{algorithm}
\usepackage{algpseudocode}
\usepackage{graphicx}
\usepackage{xcolor}
\usepackage{multirow}
\usepackage{array}
\usepackage{siunitx}
\usepackage[table]{xcolor}
\usepackage{colortbl}
\usepackage{caption}
\usepackage{pifont}
\usepackage{wrapfig}
\usepackage{caption}

\usepackage{tikz}
\usepackage{hyperref}

\usetikzlibrary{arrows.meta,decorations.pathreplacing}
\usetikzlibrary{arrows.meta, calc, patterns, decorations.pathreplacing}
\definecolor{spanfill}{RGB}{120,170,230}
\definecolor{spanfill2}{RGB}{255,200,120}
\definecolor{greylight}{RGB}{170,170,170}
\definecolor{keyblue}{RGB}{30,100,190}
\definecolor{valred}{RGB}{200,60,60}
\definecolor{residgreen}{RGB}{30,150,80}
\definecolor{projgrey}{RGB}{130,130,130}
\definecolor{bgblue}{RGB}{235,242,255}

\usepackage{booktabs}
\usepackage[table]{xcolor}
\usepackage{array}
\usepackage{amsmath}
\usepackage{siunitx}

\usepackage{wrapfig}
\usepackage{float}
\usepackage{booktabs}
\usepackage{multirow}
\usepackage{graphicx}
\usepackage[table]{xcolor}

\definecolor{oursrow}{HTML}{F3F7EE}
\definecolor{bestgreen}{HTML}{1F7A4D}
\definecolor{worstred}{HTML}{B03A3A}
\definecolor{oursedge}{HTML}{4A8B3A}
\definecolor{StepBlue}{HTML}{D9E2F3}
\definecolor{ChosenYellow}{HTML}{FCEFCB}
\definecolor{mygreen}{RGB}{0,140,0}

\newcommand{\best}[1]{\textbf{\textcolor{bestgreen}{#1}}}

\newcommand{\R}{\mathbb{R}}
\newcommand{\kv}{\mathbf{k}}
\newcommand{\vv}{\mathbf{v}}

\newcommand{\rv}{\mathbf{r}}

\newcommand{\Rset}{\mathcal{R}}
\newcommand{\Oset}{\mathcal{O}}
\newcommand{\St}{\mathcal{S}_t}
\newcommand{\norm}[1]{\left\|#1\right\|}

\newcommand{\proj}{\Pi}
\newcommand{\spanop}{\mathrm{span}}
\newcommand{\argmax}{\operatorname*{arg\,max}}
\newcommand{\argmin}{\operatorname*{arg\,min}}

\newtheorem{theorem}{Theorem}
\newtheorem{proposition}[theorem]{Proposition}

\title{
CoRDS: Coreset-based Representative and Diverse Selection for Streaming Video Understanding
\thanks{Code is publicly available at \url{https://github.com/ailarmhz/CoRDS.git}.}
}

\author{
  Ailar Mahdizadeh$^{1,2}$ \quad
  Puria Azadi$^{1}$ \quad
  Muchen Li$^{1,2}$ \quad
  Xiangteng He$^{1,2}$ \quad
  Leonid Sigal$^{1,2}$ \\
  $^{1}$University of British Columbia \\
  $^{2}$Vector Institute \\
  \texttt{ailar.mahdizadeh@ubc.ca}
}

\begin{document}
\maketitle

%% ============================================================

\begin{abstract}
Streaming video understanding with large vision-language models (VLMs) requires a compact memory that can support future reasoning over an ever-growing visual history. 
A common solution is to compress the key--value (KV) cache, but existing streaming methods typically rely on local token-wise heuristics, such as recency, temporal redundancy, or saliency, which do not explicitly optimize whether the retained cache is representative of the accumulated history. 
We propose to view KV-cache compression as a \emph{coreset selection} problem: rather than scoring tokens independently for retention, we select a small subset that covers the geometry of the accumulated visual cache. 
Our method operates in a joint KV representation and introduces a bicriteria objective that balances coverage in key and value spaces, preserving both retrieval structure and output-relevant information. 
To encourage a more diverse retained subset, we further introduce an orthogonality-driven diversity criterion that favors candidates contributing new directions beyond the current selection, and connect this criterion to log-determinant subset selection. 
Across four open-source VLMs and five long-video and streaming-video benchmarks, our method improves over heuristic streaming compression baselines under a fixed cache budget.
These results highlight that representative coreset selection offers a more effective principle, than token-wise pruning, for memory-constrained streaming video understanding.
\end{abstract}
%% ============================================================
\section{Introduction}
\label{sec:intro}

Large vision-language models (VLMs) have rapidly expanded the scope of visual
reasoning, enabling a single multimodal backbone to handle open-ended queries
over images, multi-image inputs, and long video sequences
\citep{li2024llavaonevision, qwen2vl}.
Recent advances in long-context modeling have further shown that extending the effective
context window of the language backbone can substantially improve long-video
reasoning. 
For example, LongVA~\citep{zhang2024longva} processes up to 2000 frames, or more than 200K
visual tokens, without video-specific retraining.
These developments make it increasingly plausible to deploy VLMs in always-on,
real-time settings such as streaming assistants, wearable companions, AR/VR
interfaces, embodied service robots, and other persistent perception agents that
must interpret video as it arrives rather than after the full clip has been
observed \citep{kim2025infinipotv}.

Streaming video understanding (SVU) differs fundamentally from conventional
offline video understanding (OVU).
In OVU, the full video is available before inference, enabling frame sampling,
token compression, or evidence retrieval with complete temporal
context; 
in some settings, these operations can even be conditioned on the downstream query
\citep{qian2024videostreaming, longvu}. In SVU, by contrast, future frames are unavailable, queries may arrive asynchronously, and all
pre-query processing must be causal and independent of the query
\citep{kim2025infinipotv, streamingbench, zhang2024flashvstream}. Thus, memory must remain bounded over unbounded streams, latency must support
interactive use, and past evidence must be retained without re-encoding.

% In SVU, by contrast, future frames are unavailable, questions may arrive asynchronously, and all
% pre-query processing must be causal and independent of the eventual query
% \citep{kim2025infinipotv, lin2024streamingbench, zhang2024flashvstream}.
% Thus, memory must remain bounded over unbounded streams, latency must support
% interactive use, and past evidence must be retained without re-encoding.

Prior work has sought to mitigate these constraints through \emph{frame sampling}
\citep{devnani2026trainingfree, tang2025adaptivekeyframesamplinglong},
\emph{input-side vision compression}
\citep{jiang2025storm, qian2024videostreaming, longvu},
and \emph{KV-cache compression}
\citep{di2025rekv, kim2025infinipotv, yang2025streammem}.
These approaches substantially improve the feasibility of long-horizon video
reasoning, yet they all face the same design tension: 
\emph{which frames, tokens, or
cached states deserve to be kept when only a tiny fraction of the visual history can
remain active?}
Existing methods typically answer this question using local token-level heuristics, such as recency, temporal redundancy, norm-based saliency, or attention-derived relevance
\citep{di2025rekv, kim2025infinipotv, yang2025streammem, tang2025adaptivekeyframesamplinglong}.
For instance, InfiniPot-V~\citep{kim2025infinipotv} scores cached KV entries using key similarity to estimate temporal redundancy and value norms to estimate saliency, respectively.
StreamMem~\citep{yang2025streammem} instead distils historical context into a small set of memory tokens via
cross-attention, effectively summarising the evicted cache rather than
selecting from it.
While computationally attractive, both families share a fundamental limitation:
they do not explicitly optimize whether the retained memory is
globally \emph{representative} of the full accumulated history.
As a result, cues that are individually weak but collectively important (static visual
details, weak complementary signals, temporally distributed structure) may be discarded when they are not locally salient at the moment of compression.

We take a different and more principled approach: instead of scoring cached tokens independently, we formulate KV-cache compression as representative coreset selection. The key question becomes which small subset best covers the accumulated visual KV cache under a fixed memory budget. We instantiate this idea in the cached KV feature space and make three contributions:

\textbf{(1) A coreset view of streaming KV-cache compression.}
We introduce a query-agnostic, coverage-oriented formulation for visual KV-cache compression, shifting the objective from local token-wise pruning to selecting a compact memory that represents the accumulated video context or memory.

\textbf{(2) Joint KV representative coverage.}
We show, empirically,  that representative selection should account for both keys and values: keys determine how future queries retrieve cached entries, while values
determine the information read out by attention. Based on this observation, we
introduce a joint KV representation and a bicriteria coverage objective
that explicitly balances approximation in key space and value space, yielding a
more representative retained cache.
% than K-only or V-only selection.

\textbf{(3) Orthogonal diversity for non-redundant memory.}
A greedy optimization of the representativeness objective, however, may yield a redundant coreset.  
Further, multiple representative subsets may have equal or similar representative power. 
To avoid wasting token budget, we introduce an
orthogonal anti-redundancy term that favors candidates contributing new
directions beyond the span of the current selection. We connect this criterion
to log-determinant subset selection, providing a principled motivation for proposed 
% using directional diversity as a 
lightweight regularizer on top of the coreset
coverage objective. 

This results in an inference-time approach that can easily be incorporated into any transformer-based model. Consequently, we illustrate it within various versions of Qwen \citep{qwen25vl, qwen2vl} and LLaVA-Next \citep{llavanext} families of models on five video benchmarks (EgoSchema \citep{egoschema}, MLVU \citep{mlvu}, Video MME \citep{videomme},  OVO-Bench\citep{ovobench} and StreamingBench \citep{streamingbench}) under various memory / token compression budgets. In all cases we illustrate substantial improvement over baseline approaches ({\em e.g.}, InfiniPot-V \citep{kim2025infinipotv} and StreamMem \citep{yang2025streammem}), especially in low memory budget regimes, and sometimes even improve on performance of the full model counterparts that have access to uncompressed memory -- an effect that accounts for the difficulty of dealing with long and highly imbalanced visual context, even when it fits in memory (a.k.a. needle-in-a-haystack \citep{needle}). 
\vspace{-0.05in}
\section{Related Work}
\label{sec:related}

\vspace{-0.1in}
{\bf Video Understanding.} Video understanding has a long history in computer vision, with work in the area dating back 20+ years. 
With the advent of deep learning, video understanding has shifted from hand-crafted representations to Transformer-based video foundation models and vision-language models.
VLM-based video solutions~\cite{qwen25vl,wang2025internvideo25,cheng2024videollama2,li2024llavaonevision,lin2023videollava} are highly expressive, as they allow free-form interaction and querying of the videos. However, scaling VLMs to long videos remains a core challenge, as memory and compute scale poorly (linearly and quadratically) with the length of the video in these transformer-based architectures. 
This is particularly challenging in streaming settings, where the length of the video may not only be long but is unknown and effectively infinite. 

\paragraph{Efficient video understanding via visual token pruning.}
To enable efficient long-video understanding, recent works have been focusing on inference time technique to improve the performance of long video understanding for VLMs without having to retrain/finetune them. 
A common strategy which lies in this scope is to reduce the number of visual tokens before they are processed by the LLM. 
This line of work operates at the input-token level and typically exploits redundancy in the visual stream. Existing methods prune or merge tokens based on temporal cues (EVS~\cite{bagrov2025evs}), feature-similarity clustering (DivPrune~\cite{alvar2025divprune}, Token Merging~\cite{bolya2023tome}), or vision-encoder saliency (LLaVA-PruMerge~\cite{shang2024prumerge}).
Despite being effective at reducing number of visual-tokens, these method usally relies solely on the single-frame visual features and relies on simple heruistics \cite{bagrov2025evs} to reduce temporal redundancy.
Also these line of work does not have a natural mechanism to preserve a pool of memory and therefore hard to adapt to the long streaming video setting.
As a result, they are complementary to our setting, where the key challenge is how to update and compress the accumulated KV cache as the video arrives over time.

\paragraph{Streaming video understanding and KV-cache compression.}
Streaming video understanding focuses on the online setting with unbounded input length. Benchmarks such as StreamingBench \cite{streamingbench} and models like StreamingVLM \cite{streamingvlm} and LongVU \cite{longvu} address this regime. While some methods operate at the frame level (e.g., AKS \cite{tang2025adaptivekeyframesamplinglong}), recent work integrates KV-cache compression into streaming pipelines. InfiniPot-V \cite{kim2025infinipotv}, ReKV \cite{di2025rekv}, and StreamMem \cite{yang2025streammem} apply heuristic-based strategies using similarity, norms, or retrieval. Despite their effectiveness, these methods lack formal guarantees on the representativeness of the retained memory. In contrast, our approach formulates KV-cache compression as a coreset selection problem with geometric grounding, offering a principled alternative for maintaining a compact yet representative memory in streaming video understanding. 

%% ============================================================
\newcommand{\kk}{\mathbf{k}}
\newcommand{\bphi}{\boldsymbol{\phi}}

\section{Method}
\label{sec:method}

\begin{figure*}[t]
    \centering
    \includegraphics[width=0.95\textwidth]{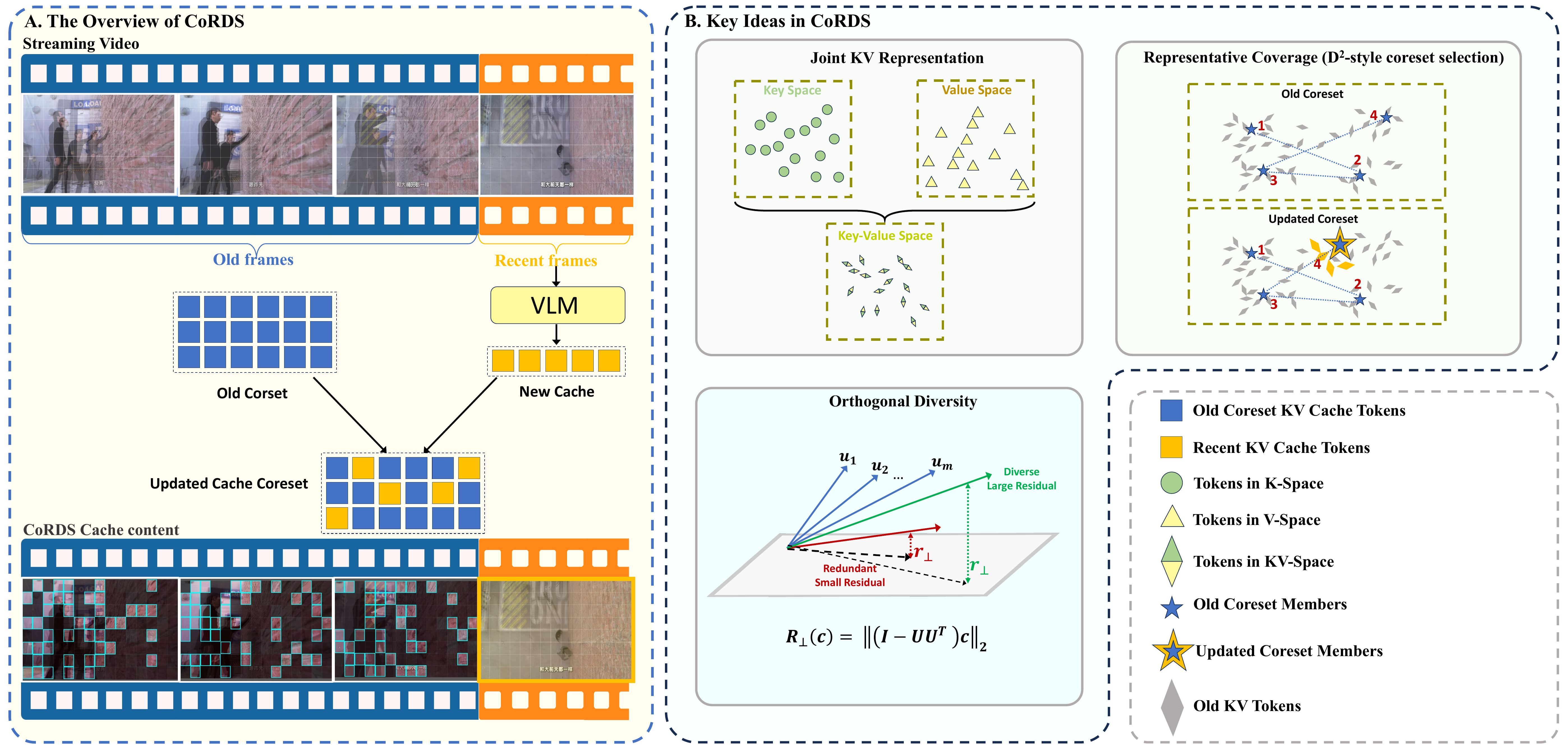}
    \caption{\textbf{CoRDS framework.} As streaming frames arrive, the accumulated KV cache is compressed by $D^2$-style coreset selection in joint K-V space (Eq.~\ref{eq:dalpha}) with an orthogonal anti-redundancy term (Eq.~\ref{eq:orth-score}); the persistent cache is then used to answer asynchronous question queries.}
    \label{fig:framework}
\end{figure*}

We formulate KV-cache compression as a representative coreset selection problem of the accumulated visual cache. In doing so, we are guided by a few core principles. First, for greatest versatility the approach should be {\em training-free}, meaning that it needs to rely on capacity of the trained model and its internal representations. Second, the selected coreset must be as {\em representative} as possible (for performance), {\em non-redundant} (for memory efficiency) and {\em query agnostic} (for practical use). Third, given hierarchy of semantic representations at different levels of the transformer, the optimal coresets may differ at these levels arguing for a potentially {\em local}, as opposed to monolithic, selection policy. 

With these design constraints in mind we first establish \emph{where selection should be made}, using the joint KV feature space; then define \emph{how to represent} tokens in that space, through a bicriteria coverage objective optimized by a deterministic \(D^2\)-style farthest-first rule; and finally address \emph{how to diversify} the selected subset, using an orthogonal anti-redundancy regularization motivated by a log-determinant approximation guarantee.

\paragraph{Setup, coreset objective, and tractable surrogate.}

A transformer decoder with $L$ layers maintains, at each layer $\ell$, key and
value projections for every processed token $i$:
\begin{equation}
  \kk_i^{(\ell)} = W_K^{(\ell)} h_i^{(\ell)},
  \qquad
  \vv_i^{(\ell)} = W_V^{(\ell)} h_i^{(\ell)},
  \label{eq:kv-proj}
\end{equation}
where $h_i^{(\ell)}\in\R^{d_{\mathrm{model}}}$ is the hidden state.
For simplicity, we omit the layer superscript when discussing the selection problem within a fixed layer.
At streaming time, the cache is naturally divided into an old cache $\Oset = \{(\kk_i,\vv_i)\}_{i=1}^{N}$ and a recent cache $\Rset$. The recent cache is always retained to preserve the latest visual state. The old cache, however, must be reduced to a budget-$b$ subset $\mathcal{S}\subset[N]$. The retained cache is therefore $\mathcal{S}\cup\Rset$, with $|\mathcal{S}|=b$.
The ideal subset would maximize downstream performance,
which depends on future queries and is intractable at selection time.

We therefore replace it with a \emph{query-agnostic geometric coverage objective}.
We view the compressed cache as a coreset: a small subset that approximates the full accumulated cache for future attention queries. Let $\bphi_i$ denote a feature representation of cached token $i$. Ideally, we want a weighted subset $(\mathcal{S},\mathbf{w})$ that uniformly approximates the full token set over a future attention queries set ($\mathcal{Q}$):
\begin{equation}
  \sup_{q\in\mathcal{Q}}
  \left|
    \sum_{i\in\mathcal{S}} w_i\, f(q,\bphi_i)
    -
    \sum_{i=1}^{N} f(q,\bphi_i)
  \right|
  \le
  \varepsilon
  \sum_{i=1}^{N} f(q,\bphi_i),
  \label{eq:coreset-def}
\end{equation}
where $f(q,\bphi_i)$ denotes the contribution of token $i$ to the
attention output, and the supremum ($\sup_{q\in\mathcal{Q}}\left| \right|$) is taken over all
$q\in\mathcal{Q}$, corresponding to the worst-case approximation error over
future queries. Here $w_i \in \mathbb{R}_{\geq 0}$ is the weight assigned to retained token $i \in \mathcal{S}$ (collected into the vector $\mathbf{w}=(w_i)_{i\in\mathcal{S}}$), reflecting how much of the original cache that token represents in the approximation. The approximation error of the compressed cache can be bounded by key-space and value-space quantization errors (Proposition~\ref{prop:attn-approx}, Appendix~\ref{app:coreset-theory}).

Since exactly solving Eq.~\eqref{eq:coreset-def} is NP-hard, we instead minimize the tractable \emph{quantization-error surrogate}
$\sum_{i\in\Oset} d_\alpha(i,\mathcal{S})$, with $d_\alpha$ defined in
Eq.~\eqref{eq:dalpha} below, which upper-bounds the coreset error in the
joint KV feature space.

When old tokens are organized into video frames, we work at frame level:
frame $t$ is represented by the centroid
$\bar{\kk}_t = \tfrac{1}{|\mathcal{I}_t|}\sum_{i\in\mathcal{I}_t}\kk_i$
(and $\bar{\vv}_t$ analogously), and selected frame indices are then expanded
to token indices.

\paragraph{Joint KV representation.}
The first design choice is the space in which coverage should be measured.
For each cached token, three natural feature choices are:
\begin{equation}
  \bphi_i^{(K)} = \kk_i,
  \qquad
  \bphi_i^{(V)} = \vv_i,
  \qquad
  \bphi_i^{(KV)} = \begin{bmatrix}\kk_i \\ \vv_i\end{bmatrix}
  \in \R^{d_k+d_v}.
  \label{eq:features}
\end{equation}
Selecting only in key space ($\bphi^{(K)}$) preserves retrieval structure but ignores the content read out by attention. Selecting only in value space ($\bphi^{(V)}$) preserves output content but ignores which cached tokens future queries are likely to retrieve. We therefore select in the joint KV representation ($\bphi^{(KV)}$). Empirically, joint KV selection consistently outperforms K-only or V-only selection (Section \ref{sec:exp-ablations-selector}); theoretically, the attention-error bound depends on both key and value approximation errors (Proposition~\ref{prop:attn-approx}, Appendix~\ref{app:coreset-theory}). 
All subsequent construction operates on $\bphi_i = \bphi_i^{(KV)}$, and the
same selected index set is used to retain the corresponding keys and values.

\paragraph{Bicriteria KV coreset.}

Directly concatenating keys ($\kk$) and values ($\vv$) implicitly fixes their relative importance through feature scale. To make this balance explicit, we define a \emph{bicriteria distance} from a candidate token $i$ to a selected set $\mathcal{S}$:
\begin{equation}
  d_\alpha(i,\mathcal{S})
  = \min_{j\in\mathcal{S}} \left[\alpha \norm{\kk_i-\kk_j}^2
    + (1-\alpha)\norm{\vv_i-\vv_j}^2 \right],
  \qquad \alpha\in[0,1].
  \label{eq:dalpha}
\end{equation}
The resulting coreset objective is:
\begin{equation}
  \mathcal{S}^\star_b
  = \argmin_{|\mathcal{S}|=b}\ \sum_{i\in\Oset} d_\alpha(i,\mathcal{S}).
  \label{eq:bicriteria-obj}
\end{equation}
The limit $\alpha\!\to\!1$ recovers K-only selection, while $\alpha\!\to\!0$ recovers V-only selection.
In practice, we find that value-favoring settings, $\alpha<0.5$, perform best, which aligns with the attention-output bound: values directly
determine output magnitude while keys only modulate which tokens are
retrieved.

\paragraph{Greedy $D^2$-style coreset selection.}

We minimize Eq.~\eqref{eq:bicriteria-obj} greedily using $D^2$-style
farthest-first selection, the deterministic counterpart of $D^2$-weighted
sampling~\citep{arthur2007kmeans}, which seeds cluster centers in proportion
to their squared distance from all previously chosen centers.

At each step $t$, the next token is
\begin{equation}
  i_{t+1}
  = \argmax_{i\in\Oset\setminus\St}\ d_\alpha\!\left(i,\, \St\right),
  \label{eq:d2-step}
\end{equation}
where $\St = \{i_1,\ldots,i_t\}$ is the current selected set.
Compared with $k$-means initialization or greedy variants that do not account for distance weighting, this squared-distance criterion is the one
that gives $D^2$ sampling its $\mathcal{O}(\log k)$-approximation guarantee
for the $k$-means objective. The full greedy pass runs in $\mathcal{O}(b\cdot N)$ per layer-head, making it
practical for online streaming inference.

\begin{wrapfigure}{r}{0.385\textwidth}
  \centering
  \vspace{-25pt}
  \begin{tikzpicture}[
    scale=0.78,
    vec/.style={-{Stealth[length=6pt,width=4pt]}, line width=1.6pt},
    dasharrow/.style={dashed, -{Stealth[length=4pt,width=3pt]}, line width=1.0pt},
    spanfillstyle/.style={fill=spanfill, opacity=0.22},
    spanfillstyle2/.style={fill=spanfill2, opacity=0.20},
    lbl/.style={font=\small},
    titlestyle/.style={font=\bfseries\small}
  ]

  % Panel (a)
  \begin{scope}[yshift=0cm]
    \node[titlestyle, anchor=north] at (2.3,2.3) {(a)~Redundant Selection};

    \fill[spanfillstyle] (0,0) -- (4.4,1.35) -- (4.4,0.45) -- cycle;
    \draw[greylight, very thin] (0,0) -- (4.4,1.35);
    \draw[greylight, very thin] (0,0) -- (4.4,0.45);
    \node[lbl, color=gray!70!black] at (3.8,0.68) {$\spanop(\mathcal{S}_t)$};

    \draw[vec, keyblue] (0,0) -- (3.8,0.9)
      node[right, lbl] {$\kv_{j_1}$};

    \coordinate (knewa) at (2.4, 0.80);
    \coordinate (kproja) at (2.451, 0.580);

    \draw[vec, valred, dashed] (0,0) -- (knewa)
      node[above, lbl, xshift=-6pt] {$\kv_{\mathrm{new}}$};

    \draw[projgrey, very thin] (knewa) -- (kproja);

    \coordinate (ra1a) at ($(kproja)!5pt!(knewa)$);
    \coordinate (ra2a) at ($(kproja)!5pt!(3.8, 0.9)$);
    \coordinate (raca) at ($(ra1a) + (ra2a) - (kproja)$);
    \draw[projgrey, thin] (ra1a) -- (raca) -- (ra2a);

    \node[lbl, valred, align=center] at (1.9,-0.65)
      {\textit{small residual} $\|\rv^K\|^2 \approx 0$};
    \node[lbl, align=center] at (2.2,-1.2) {\textit{wastes budget}};

    \fill[black] (0,0) circle (2.5pt);
    \node[lbl, below left] at (0,0) {$\mathbf{0}$};
  \end{scope}

  % separator
  \draw[greylight, very thin, dashed] (-0.3,-1.85) -- (4.9,-1.85);

  % Panel (b)
  \begin{scope}[yshift=-6.2cm]
    \node[titlestyle, anchor=north] at (2.3,4.3) {(b)~Orthogonal Selection};

    \fill[spanfillstyle] (0,0) -- (4.4,1.35) -- (4.4,0.45) -- cycle;
    \draw[greylight, very thin] (0,0) -- (4.4,1.35);
    \draw[greylight, very thin] (0,0) -- (4.4,0.45);
    \node[lbl, color=gray!70!black] at (3.8,0.68) {$\spanop(\mathcal{S}_t)$};

    \draw[vec, keyblue] (0,0) -- (3.8,0.9)
      node[right, lbl] {$\kv_{j_1}$};

    \coordinate (knewb) at (0.7, 3.2);
    \coordinate (kprojb) at (1.380, 0.327);

    \draw[dasharrow, projgrey] (0,0) -- (kprojb)
      node[below, lbl, color=projgrey, yshift=-3pt]
        {$\proj_{\mathcal{S}_t}^{K}\!\kv_{\mathrm{new}}$};

    \draw[vec, residgreen] (kprojb) -- (knewb)
      node[right, lbl, color=residgreen, align=left]
        {$\rv^K = \kv_{\mathrm{new}} - \proj^K\kv_{\mathrm{new}}$};

    \draw[vec, valred] (0,0) -- (knewb)
      node[above left, lbl] {$\kv_{\mathrm{new}}$};

    \coordinate (ra1b) at ($(kprojb)!5pt!(knewb)$);
    \coordinate (ra2b) at ($(kprojb)!5pt!(3.8, 0.9)$);
    \coordinate (racb) at ($(ra1b) + (ra2b) - (kprojb)$);
    \draw[projgrey, thin] (ra1b) -- (racb) -- (ra2b);

    \fill[spanfillstyle2] (0,0) -- (3.8,0.9) -- (knewb) -- cycle;
    \node[lbl, color=orange!75!black, rotate=52] at (1.55, 1.85)
      {\textit{new span}};

    \node[lbl, residgreen, align=center] at (1.9,-0.65)
      {\textit{large residual} $\|\rv^K\|^2 \gg 0$};
    \node[lbl, align=center] at (1.9,-1.2) {\textit{expands coverage}};

    \fill[black] (0,0) circle (2.5pt);
    \node[lbl, below left] at (0,0) {$\mathbf{0}$};
  \end{scope}

  \end{tikzpicture}
  \caption{%
  \textbf{Orthogonality intuition.} \textit{(a)} A near-span candidate has a small residual $\rv^K$ and adds little new coverage, whereas \textit{(b)} a large-residual candidate expands the span into a new region.
  }
  \label{fig:orth-intuition}
  % \vspace{-8pt}
\end{wrapfigure}
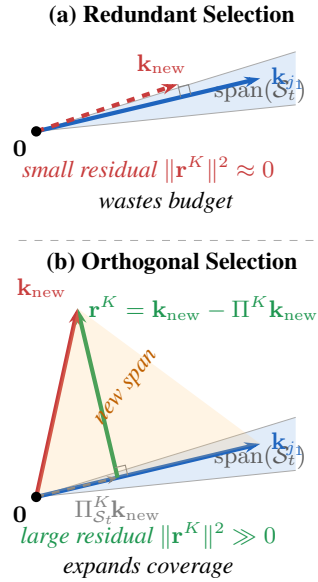

\paragraph{Orthogonal anti-redundancy.}
Coverage-based selection has a subtle failure mode: two candidates with high
$d_\alpha$ scores may lie in nearly the same direction in feature space, so
selecting both wastes budget while leaving other directions uncovered
(see Figure~\ref{fig:orth-intuition}).
We address this through an \emph{orthogonal subspace criterion} grounded in log-determinant coverage.

Given the selected key matrix
$\mathbf{U}_t = [\kk_j]_{j\in\St}\in\R^{d_k\times t}$,
define the orthogonal projector
\begin{equation}
  P_t^K
  = \mathbf{U}_t\bigl(\mathbf{U}_t^\top\mathbf{U}_t\bigr)^{\!\dagger}\mathbf{U}_t^\top
  \label{eq:projector}
\end{equation}
and the \emph{key-space residual} of candidate $i$:
\begin{equation}
  \rv_i = \bigl(I - P_t^K\bigr)\kk_i,
  \qquad
  \rv_i^V = \bigl(I - P_t^V\bigr)\vv_i.
  \label{eq:residuals}
\end{equation}
The squared norm $\norm{\rv_i}^2$ measures the component of $\kk_i$ that lies
\emph{outside} the span already covered by $\St$:
a large residual means token $i$ contributes a genuinely new direction.
Similarly, $P_t^V$ and $\rv_i^V$ are defined using the selected value matrix.
The \emph{orthogonal novelty score} is
\begin{equation}
  \mathrm{Orth}(i\mid\St)
  = \eta\,\norm{\rv_i}^2 + (1-\eta)\,\norm{\rv_i^V}^2,
  \qquad \eta\in[0,1].
  \label{eq:orth-score}
\end{equation}

This orthogonal term is motivated by classical log-determinant subset
selection. In particular, for the key matrix
$\mathbf{U}_{\mathcal{S}}^K=[\kk_j]_{j\in\mathcal{S}}$, define
\[
  \mathcal{F}(\mathcal{S})
  =
  \log\det\!\left(
    (\mathbf{U}_{\mathcal{S}}^K)^\top
    \mathbf{U}_{\mathcal{S}}^K
    +
    \epsilon I
  \right).
\]
Greedily selecting the candidate with the largest orthogonal residual maximizes
the marginal gain of this log-det objective and achieves the $(1-e^{-1})$ approximation guarantee for monotone submodular maximization (Theorem~\ref{thm:logdet} in Appendix~\ref{app:coreset-theory}). We use this result as a motivation for
the diversity term. The final algorithm, however, uses orthogonality only as a
soft bonus on top of the bicriteria coreset objective; we do not claim the same
approximation guarantee for the combined score.
\paragraph{Final selection score.}

Pure orthogonal selection is not sufficient: a geometrically novel token may be irrelevant to the model output, and exact span projection is expensive for large caches. We therefore keep bicriteria coverage as the backbone and use orthogonality as a lightweight anti-redundancy regularizer.
At step $t$, each candidate is scored, and the next representative is selected as follows:

\begin{equation}
  \mathrm{score}(i\mid\mathcal{S}_t)
  =
  \widetilde{d}_\alpha(i,\mathcal{S}_t)
  +
  \lambda
  \widetilde{\mathrm{Orth}}(i\mid\mathcal{S}_t), \;\;\;\;
  i_{t+1}
  =
  \argmax_{i\in\Oset\setminus\mathcal{S}_t}
  \mathrm{score}(i\mid\mathcal{S}_t),
  \label{eq:final-score}
\end{equation}
where $\lambda > 0$ is a small fixed weight and $\widetilde{(\cdot)}$ denotes min-max normalization over the current candidate pool. The min-max normalization is important because the bicriteria distance and orthogonal novelty can have different numerical scales. The orthogonal bonus is implemented with a max-cosine redundancy surrogate
(cheaper than exact span projection), making the per-step overhead comparable
to a single matrix-vector product. Eq.~\eqref{eq:final-score} provides a compact selection criterion that balances coverage of the joint $(\kk,\vv)$ geometry with an orthogonal diversity term that discourages redundant directions, thereby retaining tokens that best preserve the geometry of the accumulated cache.

\paragraph{Layer-selective compression and cascade reuse.}
The selection rule above can be applied independently to each decoder layer.
However, not all layers benefit equally from compression. Early layers tend to retain more redundant frame-level visual structure, while later layers contain more task-conditioned features that are more sensitive to pruning. 
We therefore apply the coreset selector only to selected lower layers and leave the remaining layers uncompressed unless otherwise specified. To reduce the cost of running selection on multiple layers, we further use a cross-layer cascade. A small set of anchor layers performs the full coreset selection, while nearby follower layers reuse the selected token indices from their anchor. Importantly, follower layers reuse only the indices, not the K/V vectors themselves. If layer $\ell'$ follows anchor layer $\ell$ with selected indices $\mathcal{S}^{(\ell)}$, then layer $\ell'$ retains its own cached values at those positions, $K^{(\ell')} \leftarrow K^{(\ell')}[\mathcal{S}^{(\ell)}]$ and $V^{(\ell')} \leftarrow V^{(\ell')}[\mathcal{S}^{(\ell)}]$. This preserves layer-specific representations while avoiding redundant selection passes.

\begin{wrapfigure}{r}{0.52\textwidth}
\vspace{-1.2em}
\begin{minipage}{\linewidth}
\begin{algorithm}[H]
\small
\caption{CoRDS: Coreset-based KV-Cache Compression}
\label{alg:coverstream}
\begin{algorithmic}[1]
\Require
  Old cache $\Oset {=} \{(\kk_i,\vv_i)\}_{i=1}^N$,
  recent cache $\Rset$ (kept),
  budget $b$,
  hyperparams $\alpha,\eta,\lambda,\varepsilon_0$
\Ensure Compressed cache $(\mathcal{S}\cup\Rset)$, $|\mathcal{S}|{=}b$
\State $\mathcal{S}_0 \!\leftarrow\! \{\argmax_i\,\norm{\kk_i{+}\vv_i}\}$
       \Comment{\scriptsize seed}
\For{$t = 0, 1, \ldots, b-1$}
  \For{each $i \in \Oset\setminus\mathcal{S}_t$}
    \State $d_\alpha(i,\mathcal{S}_t)$ via Eq.~\eqref{eq:dalpha}
    \State $\mathrm{Orth}(i\!\mid\!\mathcal{S}_t)$ via Eq.~\eqref{eq:orth-score}
  \EndFor
  \State Min-max normalize both terms
  \State $i_{t+1} \!\leftarrow\! \argmax_{i\in\Oset\setminus\mathcal{S}_t}
         \mathrm{score}(i\!\mid\!\mathcal{S}_t)$ \,(Eq.~\eqref{eq:final-score})
  \State $\mathcal{S}_{t+1} \!\leftarrow\! \mathcal{S}_t \cup \{i_{t+1}\}$
\EndFor
\State \Return $\mathcal{S}_b \cup \Rset$
\end{algorithmic}
\end{algorithm}
\end{minipage}
\vspace{-1.2em}
\end{wrapfigure}

%% ============================================================
\section{Experiments}
\label{sec:experiments}

\subsection{Experimental Setup: Models, Benchmarks, and Budgets}
\label{sec:exp-setup}
% \paragraph{Models, benchmarks, and budgets.}
We evaluate our training-free, query-agnostic compressor on four open VLMs:
Qwen2-VL-7B \citep{qwen2vl}, Qwen2.5-VL-3B/7B \citep{qwen25vl}, and
LLaVA-NeXT-Video-7B \citep{llavanext}. Following the OVU/SVU split of recent
streaming-oriented work \citep{kim2025infinipotv,yang2025streammem}, we evaluate on three
\emph{offline} long-video benchmarks (EgoSchema \citep{egoschema},
MLVU \citep{mlvu}, and VideoMME \citep{videomme}, reported
on the Short/Medium/Long splits) and two \emph{streaming} benchmarks, OVO-Bench \citep{ovobench} and StreamingBench \citep{streamingbench}, where queries arrive at arbitrary timestamps. Visual tokens are processed causally and the compression rule never accesses the question.
We use each model's standard sampling: Qwen2(.5)-VL with 768 frames
($\sim$50K tokens) and LLaVA-NeXT-Video with 128 frames ($\sim$25K tokens).
Accuracy is reported at two cache budgets, $|M|{=}3\text{K}$ and
$|M|{=}6\text{K}$ visual KV tokens.
For methods whose official results are available under matching backbone and budget, we quote numbers from the original papers; otherwise, we rerun them in our evaluation pipeline.

\subsection{Offline Video Understanding}
\label{sec:exp-offline}
For the offline long-video benchmarks, we use EgoSchema~\citep{egoschema},
MLVU~\citep{mlvu}, and VideoMME~\citep{videomme} as controlled testbeds for
streaming-style compression. Although these benchmarks are commonly evaluated
in an offline setting, we process visual tokens causally as the video unfolds,
rather than assuming that the entire video is available upfront. The cache is
therefore compressed online before the question is accessed, preserving the
query-agnostic setting used throughout our evaluation.
We compare against three families of baselines on the offline benchmarks:
the uncompressed Full KV cache as a conceptual upper bound\footnote{Note, that this is not a true upper bound in terms of performance as long context may actually negatively impact accuracy, resulting in compressed cache model performing better in practice.}; the query-agnostic
vision-token compressors STC \citep{longvu} and TTC \citep{dycoke}; and
the KV-cache compressors InfiniPot-V \citep{kim2025infinipotv} and
StreamMem \citep{yang2025streammem}.  
We include reference
points from GPT-4V/4o and proprietary long-video systems
(LLaVA-OV \citep{llavaov}, LongVU \citep{longvu}) where reported. Table~\ref{tab:ovu_main} reports accuracy on EgoSchema, MLVU, and VideoMME across four backbones at two cache budgets. At $|M|{=}6\text{K}$ ($\sim$$8\times$ compression), our method matches or exceeds the uncompressed Full KV baseline on nearly all Qwen2(.5)-VL backbones, while at $|M|{=}3\text{K}$ ($\sim$$17\times$ compression) it still outperforms InfiniPot-V at twice the budget. 
Overall, our method improves over prior compression baselines under matched budgets across the evaluated backbones, and in many cases matches or exceeds the Full KV reference. This suggests that the performance gains come from retaining a more informative and less redundant subset of visual KV tokens, rather than merely increasing the cache size.

\begin{table*}[t]
\centering
\small
\setlength{\tabcolsep}{6pt}
\renewcommand{\arraystretch}{1.08}
\resizebox{\linewidth}{!}{
\begin{tabular}{llllllll}
\toprule
\multicolumn{2}{l}{\textbf{Method}} & \textbf{Size} & \textbf{\# Frames} & \textbf{Budget} & \textbf{EgoSchema} & \textbf{MLVU} & \textbf{VideoMME} \\
& & & \textbf{(Max)} & $\mathbf{|M|}$ & & & \\
\midrule
\multicolumn{2}{l}{GPT4-V} & -- & 1fps & -- & 55.6 & -- & 60.7 \\
\multicolumn{2}{l}{GPT4-o} & -- & 1fps & -- & 72.2 & 66.2 & 77.2 \\
\midrule
\multicolumn{2}{l}{LLaVA-OV \citep{llavaov}} & 7B & 32 & 8K & 60.1 & 64.7 & 58.2 \\
\multicolumn{2}{l}{LongVU \citep{longvu}} & 7B & 1fps & 8K & 67.6 & 65.4 & 60.6 \\
\multicolumn{2}{l}{LongVU \citep{longvu}} & 3B & 1fps & 8K & 59.1 & 55.9 & 51.5 \\
\midrule

\multirow{10}{*}{\rotatebox{90}{Qwen2-VL}} 
& Original & 7B & 768 & 50K & 65.2 & 65.8 & 63.9 \\
& + TTC \citep{dycoke} & 7B & 768 & 6K & -- & 58.4 & 59.8 \\
& + STC \citep{longvu} & 7B & 768 & 6K & -- & 57.9 & 60.1 \\
& + InfiniPot-V \citep{kim2025infinipotv} & 7B & 768 & 6K & 65.6 & 65.8 & 62.8 \\
& + StreamMem \citep{yang2025streammem} & 7B & 4.0/0.5 fps & 6K & 67.2 & 65.9 & 62.1 \\
& \cellcolor{yellow!15}+ \textbf{CoRDS (ours)} & \cellcolor{yellow!15}7B & \cellcolor{yellow!15}768 & \cellcolor{yellow!15}6K & \cellcolor{yellow!15}\textbf{68.4}\textsubscript{ {\textcolor{mygreen}{+1.2}}} & \cellcolor{yellow!15}\textbf{70.1}\textsubscript{ {\textcolor{mygreen}{+4.2}}} & \cellcolor{yellow!15}\textbf{64.8}\textsubscript{ {\textcolor{mygreen}{+2.0}}} \\

& + TTC \citep{dycoke} & 7B & 768 & 3K & -- & 54.8 & 55.3 \\
& + STC \citep{longvu} & 7B & 768 & 3K & -- & 55.0 & 56.1 \\
& + InfiniPot-V \citep{kim2025infinipotv} & 7B & 768 & 3K & -- & 63.1 & 61.2 \\
& \cellcolor{yellow!15}+ \textbf{CoRDS (ours)} & \cellcolor{yellow!15}7B & \cellcolor{yellow!15}768 & \cellcolor{yellow!15}3K & \cellcolor{yellow!15}\textbf{68.1} & \cellcolor{yellow!15}\textbf{67.4}\textsubscript{ {\textcolor{mygreen}{+4.3}}} & \cellcolor{yellow!15}\textbf{64.3}\textsubscript{ {\textcolor{mygreen}{+3.1}}} \\

\midrule

\multirow{4}{*}{\rotatebox{90}{\shortstack{\footnotesize LLaVA-Next\\-Video}}} 
& Original & 7B & 128 & 25K & 67.6 & 68.7 & 62.8 \\
& + InfiniPot-V \citep{kim2025infinipotv} & 7B & 128 & 6K & 65.8 & 65.2 & 61.1 \\
& \cellcolor{yellow!15}+ \textbf{CoRDS (ours)} & \cellcolor{yellow!15}7B & \cellcolor{yellow!15}128 & \cellcolor{yellow!15}6K & \cellcolor{yellow!15}\textbf{66.2}\textsubscript{ {\textcolor{mygreen}{+0.4}}} & \cellcolor{yellow!15}\textbf{68.1}\textsubscript{ {\textcolor{mygreen}{+2.9}}}  & \cellcolor{yellow!15}\textbf{62.5}\textsubscript{ {\textcolor{mygreen}{+1.4}}} \\
& \cellcolor{yellow!15}+ \textbf{CoRDS (ours)} & \cellcolor{yellow!15}7B & \cellcolor{yellow!15}128 & \cellcolor{yellow!15}3K & \cellcolor{yellow!15}\textbf{66.0} & \cellcolor{yellow!15}\textbf{67.8} & \cellcolor{yellow!15}\textbf{62.0} \\
\midrule

\multirow{5}{*}{\rotatebox{90}{Qwen2.5-VL }} 
& Original & 3B & 768 & 50K & 64.4 & 63.3 & 60.3 \\
& + InfiniPot-V \citep{kim2025infinipotv} & 3B & 768 & 6K & 61.8 & 62.1 & 59.3 \\
& + StreamMem \citep{yang2025streammem} & 3B & 4.0/0.5 fps & 6K & 62.2 & 62.3 & 59.5 \\
& \cellcolor{yellow!15}+ \textbf{CoRDS (ours)} & \cellcolor{yellow!15}3B & \cellcolor{yellow!15}768 & \cellcolor{yellow!15}6K & \cellcolor{yellow!15}\textbf{64.1}\textsubscript{ {\textcolor{mygreen}{+1.9}}} & \cellcolor{yellow!15}\textbf{69.6}\textsubscript{ {\textcolor{mygreen}{+7.3}}} & \cellcolor{yellow!15}\textbf{61.1}\textsubscript{ {\textcolor{mygreen}{+1.6}}} \\
& \cellcolor{yellow!15}+ \textbf{CoRDS (ours)} & \cellcolor{yellow!15}3B & \cellcolor{yellow!15}768 & \cellcolor{yellow!15}3K & \cellcolor{yellow!15}\textbf{63.4} & \cellcolor{yellow!15}\textbf{68.3} & \cellcolor{yellow!15}\textbf{61.7} \\
\midrule

\multirow{3}{*}{\rotatebox{90}{\footnotesize\shortstack{Qwen2.5\\-VL}}} &Original  & 7B & 768 & 50K & 62.8 & 68.5 & 64.9 \\
& \cellcolor{yellow!15}+ \textbf{CoRDS (ours)} & \cellcolor{yellow!15}7B & \cellcolor{yellow!15}768 & \cellcolor{yellow!15}6K & \cellcolor{yellow!15}\textbf{66.4} & \cellcolor{yellow!15}\textbf{71.5} & \cellcolor{yellow!15}\textbf{65.7} \\
& \cellcolor{yellow!15}+ \textbf{CoRDS (ours)} & \cellcolor{yellow!15}7B & \cellcolor{yellow!15}768 & \cellcolor{yellow!15}3K & \cellcolor{yellow!15}\textbf{64.8} & \cellcolor{yellow!15}\textbf{70.0} & \cellcolor{yellow!15}\textbf{65.3} \\
\bottomrule
\end{tabular}}
\caption{%
  \textbf{Offline video understanding benchmarks.} \textcolor{mygreen}{Green} numbers indicate the absolute improvement over the closest compressed competitor under the same backbone and cache budget. 
}
\label{tab:ovu_main}

\end{table*}

\subsection{Streaming Video Understanding}
\label{sec:exp-streaming}
Streaming evaluation is the strictest test of a query-agnostic compressor: the question arrives at an arbitrary timestamp, so query-dependent rules cannot affect tokens compressed before the question is known. Table~\ref{tab:integrated_ovo_streaming_qwen2vl}
reports OVO-Bench and StreamingBench with Qwen2-VL-7B at
$|M|{=}6\text{K}$. The uncompressed Qwen2-VL-7B aggregates are taken from the OVO-Bench \citep{ovobench}, while InfiniPot-V
\citep{kim2025infinipotv} is re-run under our setting at matched budget. On OVO-Bench, our method exceeds both baselines on the overall score, with gains distributed across Backward, Real-time, and Forward sub-tracks rather than concentrated in one regime. On StreamingBench, our compressor exceeds InfiniPot-V on the majority of sub-tasks, with the largest gains on categories that require evidence dispersed across long stretches such as Prospective Reasoning, Counting, and Causal Reasoning.

\definecolor{headergray}{RGB}{240, 240, 240}
\definecolor{avgcol}{RGB}{252, 247, 220}
\definecolor{bestcell}{RGB}{218, 237, 218}

\newcommand{\rot}[1]{\rotatebox{60}{\scriptsize #1}}

\begin{table*}[t]
\centering
\scriptsize
\setlength{\tabcolsep}{3pt}
\renewcommand{\arraystretch}{1.1}
\caption{Integrated comparison on OVO-Bench and StreamingBench using \texttt{Qwen2-VL-7B} under a 6K compression budget. Qwen2-VL-7B aggregates are taken from the OVO-Bench\citep{ovobench}.}
\label{tab:integrated_ovo_streaming_qwen2vl}
\vspace{2pt}

% ===== OVO-Bench (transposed) =====
\begin{tabular}{@{}l ccc >{\columncolor{avgcol}}c cccccc >{\columncolor{avgcol}}c ccc >{\columncolor{avgcol}}c >{\columncolor{avgcol}}c@{}}
\toprule
& \multicolumn{4}{c}{\textbf{Backward}} & \multicolumn{7}{c}{\textbf{Real-time}} & \multicolumn{4}{c}{\textbf{Forward}} & \\
\cmidrule(lr){2-5}\cmidrule(lr){6-12}\cmidrule(lr){13-16}
\textbf{Method}
& \rot{ASI} & \rot{EPM} & \rot{HLD} & \textbf{Avg}
& \rot{ACR} & \rot{ATR} & \rot{FPD} & \rot{OCR} & \rot{OJR} & \rot{STU} & \textbf{Avg}
& \rot{REC} & \rot{CRR} & \rot{SSR} & \textbf{Avg}
& \textbf{Overall} \\
\midrule
Qwen2-VL-7B & -- & -- & -- & 46.46 & -- & -- & -- & -- & -- & -- & 55.98 & -- & -- & -- & 48.74 & 50.39 \\
InfiniPot-V  & 58.11 & 51.85 & 39.78 & 49.91 & 46.79 & 64.66 & 72.28 & 63.76 & 57.61 & 44.38 & 60.94 & 32.38 & 58.75 & 65.98 & 49.18 & 53.34 \\
\textbf{CoRDS (ours)} & 58.69 & 53.06 & 39.64 & \cellcolor{bestcell}\textbf{50.46}
              & 47.26 & 64.43 & 73.00 & 64.40 & 55.98 & 42.55 & \cellcolor{bestcell}\textbf{61.55}
              & 32.85 & 55.55 & 65.67 & \cellcolor{bestcell}\textbf{49.26}
              & \cellcolor{bestcell}\textbf{53.76} \\
\bottomrule
\end{tabular}

\vspace{3pt}
\begin{tabular}{@{}l cccccccccc >{\columncolor{avgcol}}c@{}}
\toprule
\multicolumn{12}{@{}l}{\textit{StreamingBench --- Real-time Visual Understanding}}\\
\textbf{Method}
& \textbf{CS} & \textbf{OR} & \textbf{AR} & \textbf{PR} & \textbf{ActR}
& \textbf{SU} & \textbf{EU} & \textbf{Cnt} & \textbf{TR} & \textbf{CaR}
& \textbf{Avg} \\
\midrule
Qwen2-VL-7B & -- & -- & -- & -- & -- & -- & -- & -- & -- & -- & -- \\
InfiniPot-V  & 84.36 & 77.78 & 79.91 & 74.44 & 68.40 & 65.54 & 77.88 & 38.82 & 74.39 & 77.68 & 72.75 \\
\textbf{CoRDS (ours)} & 85.97 & 81.02 & 78.79 & 81.92 & 67.82 & 63.96 & 77.82 & 40.93 & 72.88 & 80.00 & \cellcolor{bestcell}\textbf{73.43} \\
\bottomrule
\end{tabular}

\vspace{2pt}
{\scriptsize\textit{StreamingBench codes:} CS = Clips Summarize, OR = Object Recog., AR = Attribute Recog., PR = Prospective Reas., ActR = Action Recog., SU = Spatial Underst., EU = Event Underst., Cnt = Counting, TR = Text-Rich Underst., CaR = Causal Reas.}
\end{table*}

\subsection{Memory Footprint and Throughput}
\label{sec:exp-memory}
Table~\ref{tab:efficiency_long} reports peak VRAM, prefill FPS, and decode
throughput on Qwen2-VL-7B at video lengths
$L \in \{300, 500, 800, 1000\}$ seconds. Full KV runs out of memory
beyond $L = 500$\,s, while both InfiniPot-V and our method maintain a
flat memory footprint across all lengths. Our base configuration
matches InfiniPot-V on memory and decode throughput. Adding the
follower strategy, in which a small set of anchor layers run the full
selection rule and the rest reuse their decisions
(Appendix, Table~\ref{tab:follower_ablation_full}), preserves the memory footprint and
overtakes InfiniPot-V on prefill FPS at every length and on decode
throughput from $L \geq 500$\,s.

\begin{table*}[t]
\centering
\caption{Efficiency across varying video durations on Qwen2-VL-7B. V: peak VRAM (GB), F: prefill FPS, T: decode throughput (tok/s). Subscripts denote std over 5 runs; V is deterministic. Full KV OOMs for $L\geq600$\,s.}
\label{tab:efficiency_long}
\footnotesize
\setlength{\tabcolsep}{3pt}
\newcommand{\sd}[1]{_{\scriptscriptstyle\pm#1}}
\resizebox{\linewidth}{!}{
\begin{tabular}{l lll lll lll lll}
\toprule
& \multicolumn{3}{c}{$L=300$\,s} & \multicolumn{3}{c}{$L=500$\,s}
& \multicolumn{3}{c}{$L=800$\,s} & \multicolumn{3}{c}{$L=1000$\,s} \\
\cmidrule(lr){2-4}\cmidrule(lr){5-7}\cmidrule(lr){8-10}\cmidrule(lr){11-13}
Method
 & V$\downarrow$ & F$\uparrow$ & T$\uparrow$
 & V$\downarrow$ & F$\uparrow$ & T$\uparrow$
 & V$\downarrow$ & F$\uparrow$ & T$\uparrow$
 & V$\downarrow$ & F$\uparrow$ & T$\uparrow$ \\
\midrule
Full KV
 & 24.4 & $32.4\sd{0.2}$ & $13.8\sd{0.6}$
 & 35.9 & $27.5\sd{0.4}$ & $\mathbf{30.5}\sd{0.4}$
 & \multicolumn{3}{c}{OOM}
 & \multicolumn{3}{c}{OOM} \\
InfiniPot-V~\cite{kim2025infinipotv}
 & \textbf{10.1} & $42.8\sd{0.6}$ & $\mathbf{23.2}\sd{0.4}$
 & \textbf{10.7} & $43.9\sd{0.5}$ & $22.0\sd{0.4}$
 & \textbf{11.6} & $43.8\sd{1.7}$ & $20.5\sd{3.0}$
 & \textbf{11.8} & $44.4\sd{1.8}$ & $22.2\sd{2.2}$ \\
\textbf{CoRDS (w/o follower)}
 & 10.2 & $36.7\sd{0.2}$ & $\mathbf{23.2}\sd{0.4}$
 & 11.6 & $33.2\sd{0.2}$ & $23.5\sd{0.2}$
 & 13.8 & $28.4\sd{0.2}$ & $\mathbf{23.4}\sd{0.3}$
 & 14.3 & $27.1\sd{0.2}$ & $22.8\sd{0.2}$ \\
\textbf{CoRDS} 
 & \textbf{10.1} & $\mathbf{43.2}\sd{0.4}$ & $22.4\sd{0.2}$
 & \textbf{10.7} & $\mathbf{45.0}\sd{0.3}$ & $22.5\sd{0.2}$
 & \textbf{11.6} & $\mathbf{45.9}\sd{0.5}$ & $23.2\sd{0.3}$
 & \textbf{11.8} & $\mathbf{45.9}\sd{0.8}$ & $\mathbf{23.8}\sd{0.2}$ \\
\bottomrule
\end{tabular}
}
\end{table*}

\begin{figure}[t]
    \centering
    \includegraphics[width=\textwidth]{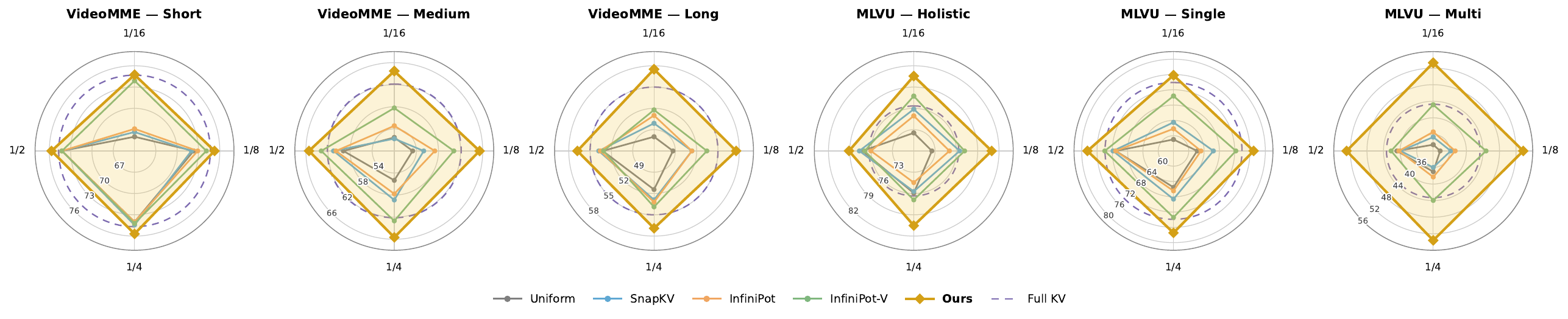}
\caption{\textbf{Per-task accuracy across compression ratios on Qwen2-VL-7B.}
Each radar compares methods across compression ratios; the dashed ring denotes Full KV.}
    \label{fig:radar_methods}
\end{figure}

\subsection{Compression Ratios}
\label{sec:exp-compression}
Figure~\ref{fig:radar_methods} sweeps the compression ratio
$\{1/16, 1/8, 1/4, 1/2\}$ on Qwen2-VL-7B across the three VideoMME
splits (Short, Medium, Long) and three MLVU task types (Holistic,
Single, Multi), comparing our method against Uniform, SnapKV,
InfiniPot, InfiniPot-V, and the Full KV reference. Our method
consistently traces the outer edge of every radar regardless of ratio,
and on several tasks lies entirely outside the Full KV ring holds across task types and
length splits rather than being an artifact of any single benchmark.
The advantage is largest at aggressive ratios ($1/16$ and $1/8$), where
all baselines visibly contract toward the centre while ours remains
flat, supporting the claim that gains come from \emph{which} tokens
are kept rather than how many. Per-ratio numerical results across all
backbones are reported in Appendix B, Table ~\ref{tab:qwen_results}.

\subsection{Qualitative Analysis}
\label{sec:exp-qualitative}
We further assess whether the retained cache is \emph{representative} of the
full accumulated KV cache, beyond downstream accuracy. For each
video, we first run the model without compression to obtain a reference full
KV cache, and then run each compression method under the same budget to record the retained tokens. For every full-cache token, we compute the
cosine distance to its nearest retained neighbour and plot the CDF (see Fig.~\ref{fig:cdf});
curves further to the left indicate better coverage. To avoid bias
toward our own selection objective, we report this CDF under three
distance metrics: joint $K\|V$ (after per-component normalisation),
$K$-only, and $V$-only. Fig.~\ref{fig:cdf} shows the result for the \emph{MLVU anomaly recognition}
task with Qwen2-VL-7B at $|M|{=}6\text{K}$; CDFs for the
remaining MLVU tasks are in Appendix~\ref{app:qualitative}. Across all
three metrics, CoRDS covers the full cache better than InfiniPot-V, with peak
gaps of $\Delta = 0.136$ (Joint $K\|V$), $0.156$ ($K$-only), and
$0.084$ ($V$-only) at small distances ($d \approx 0.08\text{--}0.12$).
Since the improvement also appears under $K$-only and $V$-only
metrics, 
the gain is unlikely to be an artifact of evaluating with the same metric used for selection. This provides a plausible explanation for the downstream gains in Sec.~\ref{sec:exp-offline}: CoRDS retains a more representative and less redundant subset of the visual KV cache.\\

\begin{figure}[t]
  \centering
  \includegraphics[width=0.8\linewidth]{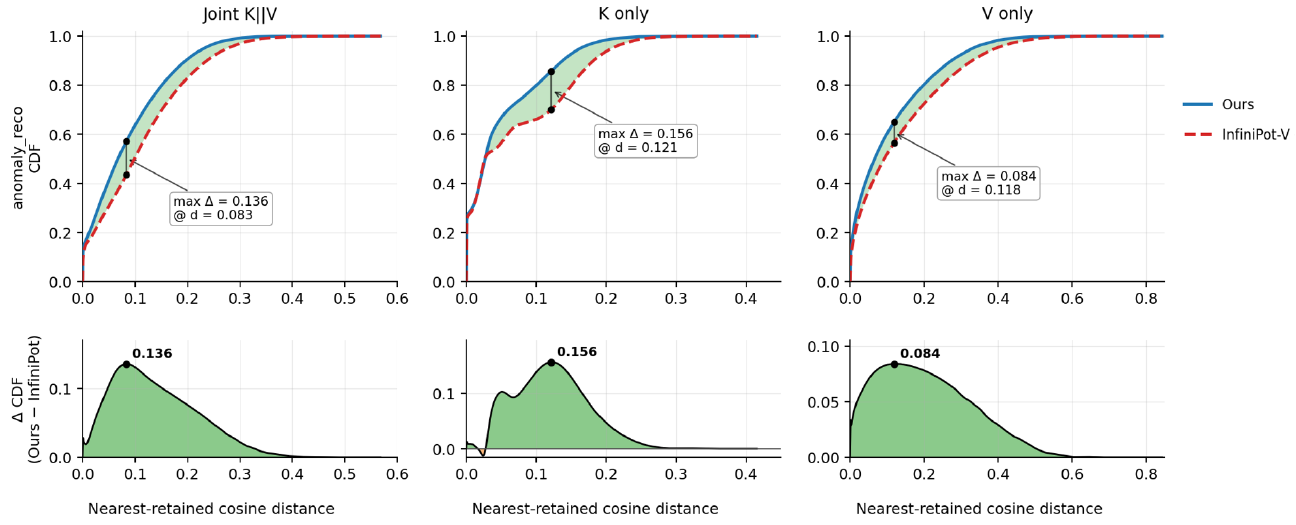}
  \caption{\textbf{Coverage CDF on MLVU anomaly recognition} (Qwen2-VL-7B, $|M|{=}6\text{K}$).
Curves further left indicate better full-cache coverage under joint $K\|V$, $K$-only, and $V$-only metrics.}
  \label{fig:cdf}
\end{figure}

\subsection{Ablations: Selector, Feature Space, and Layer Subset}
\label{sec:exp-ablations-selector}
 
\captionsetup{font=small}
\renewcommand{\arraystretch}{0.95}
 
\newcommand{\stepbar}[2]{%
  \rowcolor{StepBlue}\multicolumn{4}{@{}l}{\textbf{Step #1.} #2}}
\newcommand{\locked}[1]{\textcolor{gray!80}{\textit{#1}}}
 
\begin{wraptable}{r}{0.42\textwidth}
\vspace{-1.2em}
\begin{minipage}{\linewidth}
  \caption{%
    \textbf{Design journey:}
  Each step locks in one choice ($\bigstar$) which feeds the next.}
  \label{tab:design-journey}
  \centering
  \small
  \setlength{\tabcolsep}{3pt}
  \setlength{\aboverulesep}{0pt}
  \setlength{\belowrulesep}{0pt}
  \setlength{\extrarowheight}{0pt}
  \begin{tabular}{@{}l l
                  S[table-format=1.3]
                  S[table-format=1.2]}
    \toprule
    & \textbf{Variant} & {\textbf{Acc.}\,$\uparrow$} & {\textbf{Time}\,$\downarrow$} \\
    \midrule
    \stepbar{(i)}{\emph{Selector}} \\
    & K-means representative           & 60.0 & 1.93$\times$ \\
    & Greedy shortlist      & 66.5 & 1.19$\times$ \\
    \rowcolor{ChosenYellow}
    $\bigstar$ & \textbf{D$^2$ style farthest-point}  & 70.0 & 1.00$\times$ \\
    \midrule
    \stepbar{(ii)}{\emph{Feature space}} \\
    & K-only ($k_i$)                   & 66.0 & 1.00$\times$ \\
    & V-only ($v_i$)                   & 66.5 & 1.07$\times$ \\
    \rowcolor{ChosenYellow}
    $\bigstar$ & \textbf{KV-joint $[k_i;v_i]$}    & 70.0 & 1.08$\times$ \\
    \midrule
    \stepbar{(iii)}{\emph{Orth.\ compression: which layers?}} \\
    & Top 50\% layers                  & 62.0 & 0.94$\times$ \\
    & Top 25\% layers                  & 74.5 & 0.81$\times$ \\
    \rowcolor{ChosenYellow}
    $\bigstar$ & \textbf{Bottom-25\% layers}&77.0 & 2.01$\times$ \\
    \bottomrule
  \end{tabular}
\end{minipage}
\vspace{-1em}
\end{wraptable}
Table~\ref{tab:design-journey} reads top to bottom as controlled
decisions, each locking in the previous winner ($\bigstar$) on a
held-out MLVU PlotQA slice. \textbf{(i) Selector:} D$^2$ style
farthest-point matches the best accuracy at the lowest wall-clock.
\textbf{(ii) Feature space:} keys-only and values-only both
underperform; KV-joint $[k_i; v_i]$ recovers full accuracy, indicating both subspaces carry complementary information. \textbf{(iii) Layer
subset:} all-layer orthogonality fails to help
(Appendix~\ref{app:ablations-extended}), so we restrict it to a subset; \emph{bottom 25\%} is the accuracy-prioritized variant. Additional ablations on $\alpha, \eta,\lambda,\varepsilon_0)$ are in Appendix \ref{app:hyperparam} (Tables \ref{tab:abl-alpha}, \ref{tab:abl-eta}, \ref{tab:abl-lambda}, and \ref{tab:abl-eps}, respectively), and follower configurations are covered in Appendix \ref{app:efficiency}, Table \ref{tab:follower_ablation_full}. 
\section{Conclusion}

We presented \textbf{CoRDS}, a training-free, query-agnostic KV-cache compressor that reframes streaming KV cache compression as representative coreset selection. CoRDS combines three ingredients: a joint KV representation justified by an attention-output error bound, a bicriteria $D^2$-style farthest-first objective balancing retrieval-side and content-side coverage, and an orthogonal anti-redundancy regularizer that discourages selecting tokens which lie in directions already spanned by the current subset. A layer-selective schedule with cross-layer cascade keeps overhead on par with existing streaming compressors. Across four open VLMs and five long-video and streaming benchmarks, CoRDS consistently outperforms heuristic baselines under matched budgets and, at aggressive compression ratios, often matches or exceeds the uncompressed Full-KV reference, indicating that \emph{which} tokens are kept matters more than \emph{how many}. Coverage-CDF diagnostics confirm that gains arise from a more representative and less redundant retained cache. Promising next steps include lightweight late-stage query conditioning, adaptive per-layer budgets, and extending the coreset view to audio-visual and embodied long-context agents.
% \clearpage
% \begin{ack}
% Omitted for anonymous submission.
% \end{ack}

{
\small
\bibliography{references.bib}
}

% \newpage

\appendix
\section{Theory and Proofs}
\label{app:coreset-theory}

This appendix provides the formal statements and proofs underlying two
ingredients of our method: (i) the bicriteria coverage objective that we
minimise as a tractable surrogate for the worst-case attention-output
deviation of the compressed cache (Proposition~\ref{prop:attn-approx}), and
(ii) the log-determinant guarantee that motivates the orthogonal
anti-redundancy term (Theorem~\ref{thm:logdet}).

\subsection{Notation and Setup}

Let $\mathcal{C} = \{(\kk_i,\vv_i)\}_{i=1}^N$ with
$\kk_i\in\R^{d_k}$, $\vv_i\in\R^{d_v}$.
The attention output for query $\mathbf{q}\in\R^{d_k}$ is
\begin{equation}
  \mathrm{Attn}(\mathbf{q},\mathcal{C})
  = \sum_{i=1}^N a_i(\mathbf{q})\,\vv_i,
  \qquad
  a_i(\mathbf{q})
  = \frac{\exp(\mathbf{q}^\top\kk_i/\sqrt{d_k})}
         {\sum_{j=1}^N \exp(\mathbf{q}^\top\kk_j/\sqrt{d_k})},
  \label{eq:app-attn}
\end{equation}
so $a_i(\mathbf{q})\ge 0$ and $\sum_i a_i(\mathbf{q}) = 1$ for every $\mathbf{q}$.
Let $\mathcal{S}\subset[N]$ with $|\mathcal{S}|=b$ be a selected subset, and
define the \emph{cluster-weighted} compressed attention:
for each $j\in\mathcal{S}$ let $\mathcal{C}_j = \{i : j^\star(i)=j\}$ where
$j^\star(i) = \argmin_{j'\in\mathcal{S}}\|\bphi_i - \bphi_{j'}\|$
is the nearest representative of token $i$ in the joint feature space
$\bphi_i = [\kk_i;\vv_i]\in\R^{d_k+d_v}$.
Assign cluster weight $w_j(\mathbf{q}) = \sum_{i\in\mathcal{C}_j} a_i(\mathbf{q})$,
so that the compressed attention output is
\begin{equation}
  \mathrm{Attn}_w(\mathbf{q},\mathcal{S})
  = \sum_{j\in\mathcal{S}} w_j(\mathbf{q})\,\vv_j.
  \label{eq:app-attn-compressed}
\end{equation}
This is the natural cluster-centroid approximation: every token $i$ contributes
its attention weight to its nearest selected representative.

\subsection{Attention Approximation Bound}

\begin{proposition}[Attention output error is bounded by the bicriteria quantization error]
  \label{prop:attn-approx}
  For any query $\mathbf{q}$ and any subset $\mathcal{S}$ with cluster assignment
  $j^\star$ defined above,
  \begin{equation}
    \bigl\|\mathrm{Attn}(\mathbf{q},\mathcal{C})
           - \mathrm{Attn}_w(\mathbf{q},\mathcal{S})\bigr\|_2
    \;\le\;
    \sqrt{\,\sum_{i=1}^N a_i(\mathbf{q})\cdot\min_{j\in\mathcal{S}}\|\vv_i - \vv_j\|^2\,}.
    \label{eq:attn-error-bound}
  \end{equation}
  Moreover, since
  $\min_{j\in\mathcal{S}}\|\vv_i-\vv_j\|^2 \le \tfrac{1}{1-\alpha}\,d_\alpha(i,\mathcal{S})$
  for any $\alpha\in[0,1)$, the bound is controlled by the bicriteria
  quantization surrogate of Eq.~\eqref{eq:dalpha} in the main text:
  \begin{equation}
    \bigl\|\mathrm{Attn}(\mathbf{q},\mathcal{C})
           - \mathrm{Attn}_w(\mathbf{q},\mathcal{S})\bigr\|_2
    \;\le\;
    \sqrt{\,\frac{1}{1-\alpha}
          \sum_{i=1}^N a_i(\mathbf{q})\cdot d_\alpha(i,\mathcal{S})\,}.
    \label{eq:attn-error-bound-dalpha}
  \end{equation}
  In particular, the right-hand side depends on \emph{both} the key-space
  and value-space quantization errors of $\mathcal{S}$, which is what
  motivates joint $(\kk,\vv)$ selection in the main text.
\end{proposition}

\begin{proof}
Write the error as a weighted sum over clusters:
\begin{align}
  \mathrm{Attn}(\mathbf{q},\mathcal{C})
  - \mathrm{Attn}_w(\mathbf{q},\mathcal{S})
  &= \sum_{i=1}^N a_i(\mathbf{q})\,\vv_i
   - \sum_{j\in\mathcal{S}} \Bigl(\sum_{i\in\mathcal{C}_j} a_i(\mathbf{q})\Bigr)\vv_j
  \nonumber\\
  &= \sum_{i=1}^N a_i(\mathbf{q})\,\bigl(\vv_i - \vv_{j^\star(i)}\bigr).
  \label{eq:app-error-decomp}
\end{align}
Taking the $\ell_2$ norm and applying the triangle inequality gives
\begin{align}
  \bigl\|\mathrm{Attn}(\mathbf{q},\mathcal{C}) - \mathrm{Attn}_w(\mathbf{q},\mathcal{S})\bigr\|_2
  &\le \sum_{i=1}^N a_i(\mathbf{q})\,\|\vv_i - \vv_{j^\star(i)}\|_2.
  \label{eq:app-triangle}
\end{align}
Cauchy--Schwarz with the probability weights $a_i(\mathbf{q})$ (which sum to
one) yields
\begin{align}
  \sum_{i=1}^N a_i(\mathbf{q})\,\|\vv_i - \vv_{j^\star(i)}\|_2
  &\le \sqrt{\sum_{i=1}^N a_i(\mathbf{q})}
       \cdot\sqrt{\sum_{i=1}^N a_i(\mathbf{q})\,\|\vv_i - \vv_{j^\star(i)}\|_2^2}
  \nonumber\\
  &= \sqrt{\sum_{i=1}^N a_i(\mathbf{q})\,\|\vv_i - \vv_{j^\star(i)}\|_2^2},
  \label{eq:app-cs}
\end{align}
where the last equality uses $\sum_i a_i(\mathbf{q})=1$.
Since $\|\vv_i - \vv_{j^\star(i)}\|^2 = \min_{j\in\mathcal{S}}\|\vv_i-\vv_j\|^2$
by definition of $j^\star$, this establishes Eq.~\eqref{eq:attn-error-bound}.

For Eq.~\eqref{eq:attn-error-bound-dalpha}, observe that
\begin{equation}
  d_\alpha(i,\mathcal{S})
  = \alpha\min_{j\in\mathcal{S}}\|\kk_i-\kk_j\|^2
    + (1-\alpha)\min_{j\in\mathcal{S}}\|\vv_i-\vv_j\|^2
  \;\ge\; (1-\alpha)\min_{j\in\mathcal{S}}\|\vv_i-\vv_j\|^2,
\end{equation}
so $\min_{j\in\mathcal{S}}\|\vv_i-\vv_j\|^2 \le d_\alpha(i,\mathcal{S})/(1-\alpha)$.
Substituting into Eq.~\eqref{eq:attn-error-bound} gives
Eq.~\eqref{eq:attn-error-bound-dalpha}, in which the right-hand side mixes
both $\|\kk_i-\kk_j\|^2$ and $\|\vv_i-\vv_j\|^2$ contributions through
$d_\alpha$.
\end{proof}

\subsection{Log-Determinant Guarantee for Greedy Orthogonalization}

We now justify the orthogonal anti-redundancy term used in
Eq.~\eqref{eq:orth-score} of the main text. The greedy rule that selects the
candidate with the largest residual norm is not an ad-hoc heuristic: it is
exactly equivalent to maximising the marginal gain of a log-determinant
coverage objective, which is monotone submodular and therefore admits a
$(1-e^{-1})$ approximation guarantee.

\begin{theorem}[Greedy orthogonalization achieves a $(1-e^{-1})$ log-det approximation]
  \label{thm:logdet}
  Define the log-determinant coverage of selected keys as
  $\mathcal{F}(\mathcal{S}) = \log\det\!\bigl(\mathbf{U}_\mathcal{S}^\top\mathbf{U}_\mathcal{S} + \epsilon I\bigr)$
  where $\mathbf{U}_\mathcal{S} = [\kk_j]_{j\in\mathcal{S}}$ and $\epsilon>0$
  is a small regularizer.
  The greedy rule
  $i_{t+1} = \argmax_{i\notin\St}\norm{\rv_i}^2$
  with $\rv_i$ defined as in Eq.~\eqref{eq:residuals} of the main text
  produces a set $\hat{\mathcal{S}}_b$ satisfying
  \begin{equation}
    \mathcal{F}(\hat{\mathcal{S}}_b)
    \;\ge\;
    \bigl(1 - e^{-1}\bigr)\,\mathcal{F}(\mathcal{S}^\star_b),
    \label{eq:logdet-bound}
  \end{equation}
  where $\mathcal{S}^\star_b = \argmax_{|\mathcal{S}|=b}\mathcal{F}(\mathcal{S})$.
\end{theorem}

\begin{proof}
We show (i) the marginal gain of adding token $i$ equals $\log(1+\norm{\rv_i}^2/c_t)$
for a positive scalar $c_t$, so the greedy rule maximizing $\norm{\rv_i}^2$
also maximizes marginal $\mathcal{F}$-gain; and (ii) $\mathcal{F}$ is monotone
submodular, so the greedy maximizer achieves the $(1-e^{-1})$ bound
of \citep{nemhauser1978}.

\textit{(i) Marginal gain equals residual norm.}
Partition the Gram matrix after adding token $i$ to $\St$ as
\[
  \mathbf{U}_{t+1}^\top\mathbf{U}_{t+1}
  = \begin{bmatrix}
      \mathbf{U}_t^\top\mathbf{U}_t & \mathbf{U}_t^\top\kk_i \\
      \kk_i^\top\mathbf{U}_t & \norm{\kk_i}^2
    \end{bmatrix}.
\]
By the matrix determinant lemma (Schur complement of the $(t,t)$ block):
\begin{align}
  &\log\det\!\bigl(\mathbf{U}_{t+1}^\top\mathbf{U}_{t+1}+\epsilon I\bigr)
    - \log\det\!\bigl(\mathbf{U}_t^\top\mathbf{U}_t+\epsilon I\bigr)
  \nonumber\\
  &\quad= \log\!\Bigl(
      \norm{\kk_i}^2 + \epsilon
      - \kk_i^\top\mathbf{U}_t
        \bigl(\mathbf{U}_t^\top\mathbf{U}_t+\epsilon I\bigr)^{-1}
        \mathbf{U}_t^\top\kk_i
    \Bigr)
    - \log\epsilon
  \nonumber\\
  &\quad= \log\!\Bigl(
      1 + \tfrac{1}{\epsilon}\,
      \kk_i^\top
      \bigl(I - P_t^{K,\epsilon}\bigr)
      \kk_i
    \Bigr),
  \label{eq:marginal-gain}
\end{align}
where $P_t^{K,\epsilon} = \mathbf{U}_t(\mathbf{U}_t^\top\mathbf{U}_t+\epsilon I)^{-1}\mathbf{U}_t^\top$
is the regularized projector that converges to $P_t^K$ as $\epsilon\to 0$.
The quadratic form
$\kk_i^\top(I - P_t^{K,\epsilon})\kk_i = \norm{(I-P_t^{K,\epsilon})\kk_i}^2$
is the squared regularized residual, which equals
$\norm{\rv_i}^2$ up to an $O(\epsilon)$ term.
Since $\log(1+x/\epsilon)$ is strictly monotone increasing in $x > 0$,
the greedy rule maximizing $\norm{\rv_i}^2$ at each step also maximizes
the exact marginal $\mathcal{F}$-gain.

\textit{(ii) Submodularity of $\mathcal{F}$.}
For any $A\subseteq B$ and token $i\notin B$, the marginal gain
$\mathcal{F}(A\cup\{i\})-\mathcal{F}(A) \ge \mathcal{F}(B\cup\{i\})-\mathcal{F}(B)$
holds because the residual of $\kk_i$ with respect to a larger selected set
can only decrease (the span of $B$ contains the span of $A$, so more of $\kk_i$
is already projected out). Formally, if $A\subseteq B$ then
$P_A \preceq P_B$ in the Loewner order, hence
$(I-P_B)\preceq(I-P_A)$ and the Schur complement in Eq.~\eqref{eq:marginal-gain}
is non-increasing as the conditioned set grows.
Monotonicity $\mathcal{F}(A)\le\mathcal{F}(B)$ for $A\subseteq B$ follows
directly from the positive semi-definiteness of the Gram matrix increment.
By the greedy submodularity theorem of \citep{nemhauser1978},
any monotone submodular function subject to a cardinality constraint is
$(1-e^{-1})$-approximated by the greedy maximizer, yielding
Eq.~\eqref{eq:logdet-bound}.
\end{proof}

% \paragraph{Implementation Detail.}
% All models are evaluated with a uniform block-wise streaming inference
% configuration: block size 32, 24 compressed frames per block, and a
% maximum of 768 sampled frames per video. We fix the compressed KV budget
% to $|M|{=}3$K tokens across all of our runs unless otherwise stated, which
% is $2\times$ more aggressive than the $|M|{=}6$K used by most prior
% compression baselines. Reported metrics are task accuracy and, where
% relevant, mean and median per-sample inference time on an L40S.
% Ablations are run on a held-out 200-sample slice of MLVU PlotQA to keep
% turnaround manageable; all main-table numbers are on the full benchmark
% test sets. 

%%%%%%%%%Ablations
\section{Extended Experimental Results}
\label{app:experiments-extended}
 
This appendix expands the experimental section of the main paper. Appendix~\ref{app:impl-extended} gives the full implementation and reproducibility details; Appendix~\ref{app:streaming-detail} gives per-task breakdowns of the streaming results;
Appendix~\ref{app:per-backbone} reports the per-backbone, per-budget numbers;
Appendix~\ref{app:hyperparam} ablates the four scalar hyperparameters $(\alpha,\eta,\lambda,\varepsilon_0)$;
Appendix~\ref{app:ablations-extended} extends the design ablations of Section~\ref{sec:exp-ablations-selector} to the full sweep;
Appendix~\ref{app:qualitative} provides per-task coverage CDFs;
Appendix~\ref{app:efficiency} extends the memory--throughput sweep;
Appendix~\ref{app:abl-cascade} covers the cross-layer cascade analysis.
 
%% ===========================================================================
\subsection{Implementation Details}
\label{app:impl-extended}

All experiments are run on a single NVIDIA L40S 48\,GB GPU with FlashAttention-2 \citep{flashattention2}, PyTorch 2.3, \texttt{transformers} 4.45, and \texttt{bfloat16} weights and activations. Wall-clock results are reported as mean per-sample times over the corresponding evaluation slice. We evaluate four public VLMs without fine-tuning: Qwen2-VL-7B \citep{qwen2vl}, Qwen2.5-VL-3B/7B \citep{qwen25vl}, and LLaVA-NeXT-Video-7B \citep{llavanext}. We use each model's standard sampling setting, with 768 frames and approximately 50K tokens for Qwen2/2.5-VL, and 128 frames and approximately 25K tokens for LLaVA-NeXT-Video. Visual tokens are fed causally, and the question is appended only after compression.

The compressor uses compressed frames to block size ratio of 0.75, and a recent tail of $|\mathcal{R}|{=}\lfloor |M|/4 \rfloor$. The selection rule in Eq.~\ref{eq:final-score} is fixed at $(\alpha,\eta,\lambda,\varepsilon_0){=}(0.25,\,0.25,\,0.25,\,10^{-6})$ and $\beta{=}0.15$, with min-max normalization. Compression is active on the bottom 25\% of decoder layers, while layers outside the active set reuse the most recent active-layer selection.

Decoding uses greedy generation with each model's default chat template and a 256-token cap for multiple-choice benchmarks. The held-out 200-sample MLVU PlotQA slice used for Table~\ref{tab:design-journey} and per-step ablations is fixed across runs, as are all random seeds.
%===========================================================================

\begin{figure}
    \centering
    \includegraphics[width=1\linewidth]{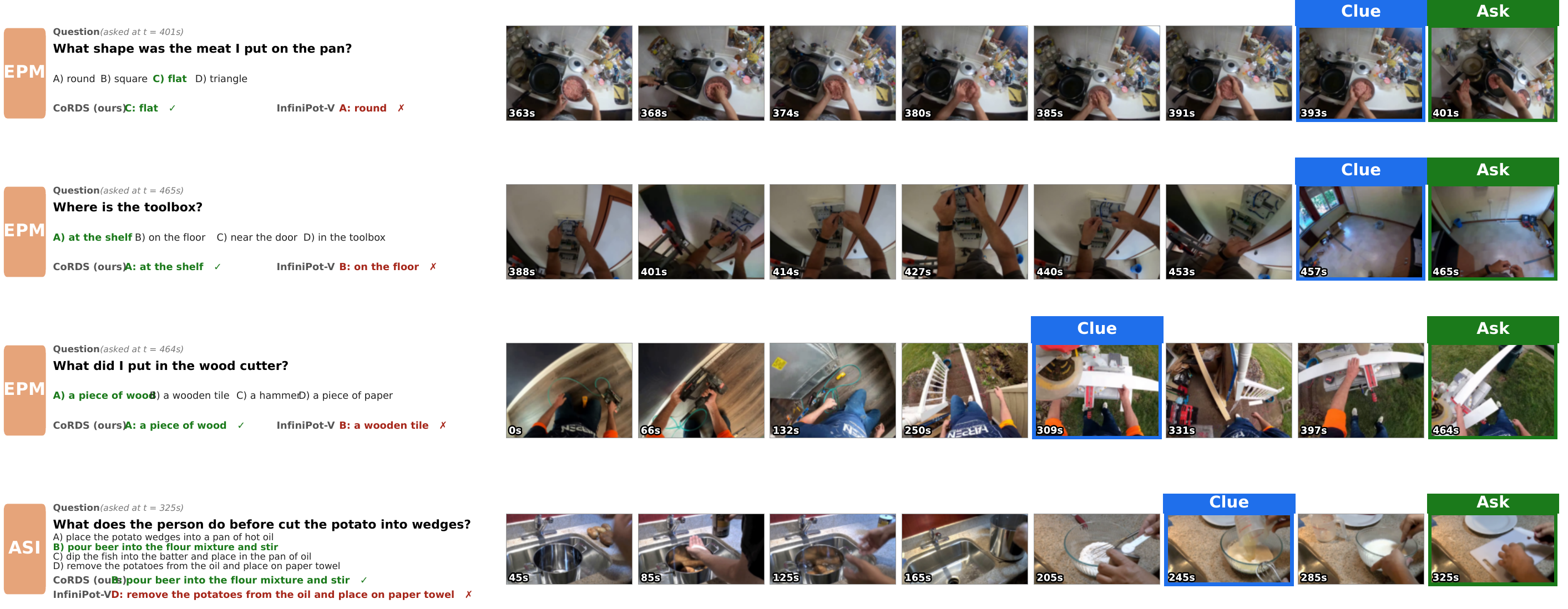}
    \caption{Backward qualitative examples on EPM/ASI: CoRDS (ours) recovers the clue laid down hundreds of seconds before the question, while InfiniPot-V drops it.}
\label{fig:BW_task}
\end{figure}

\begin{figure}
    \centering
    \includegraphics[width=1\linewidth]{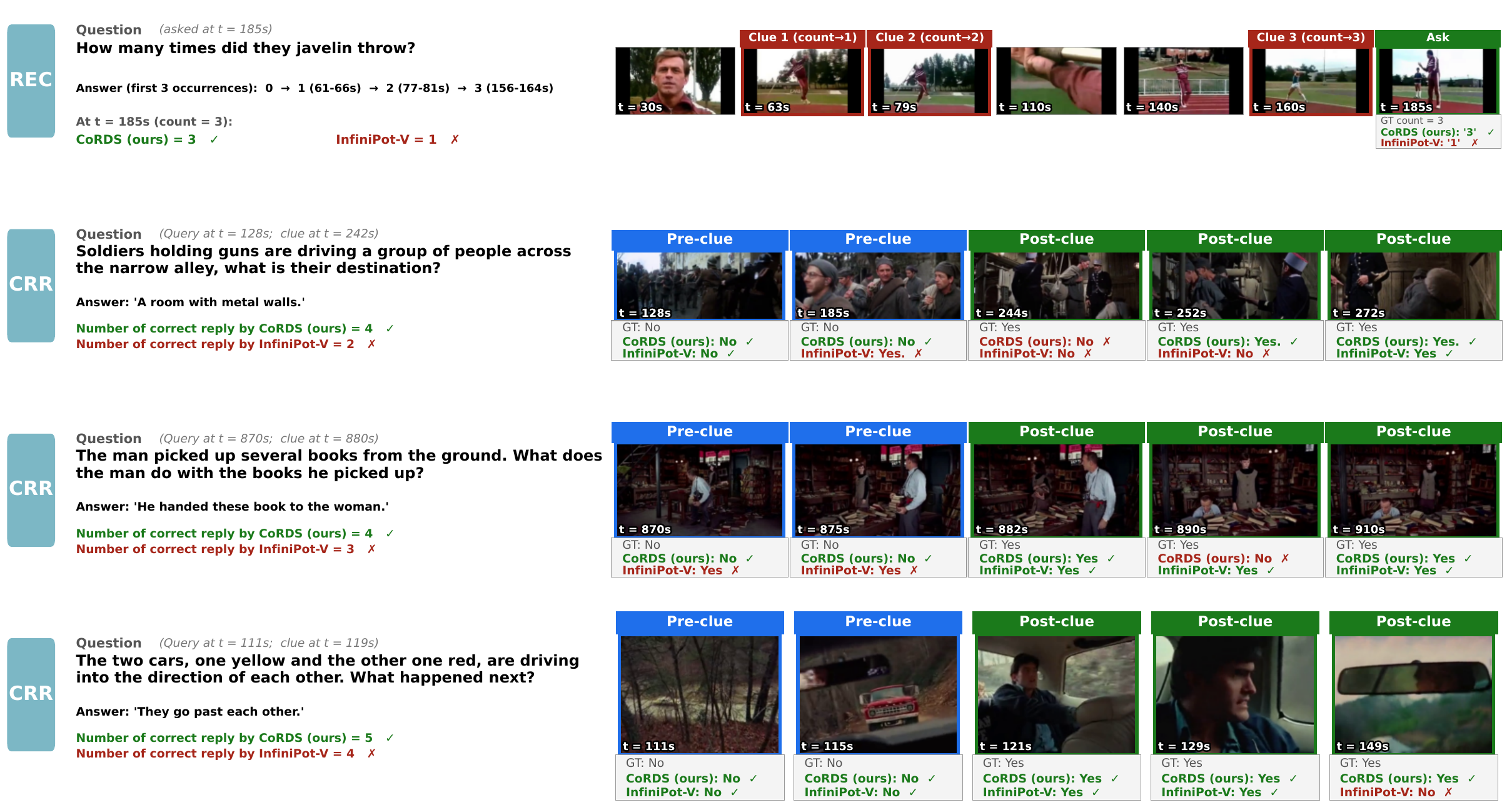}
    \caption{Forward qualitative examples on REC/CRR: CoRDS (ours) accumulates the three javelin-throw occurrences and tracks pre-/post-clue answers across time, while InfiniPot-V undercounts and answers affirmatively before the clue.}
\label{fig:FW_task}

\end{figure}

\begin{figure}
    \centering
    \includegraphics[width=1\linewidth]{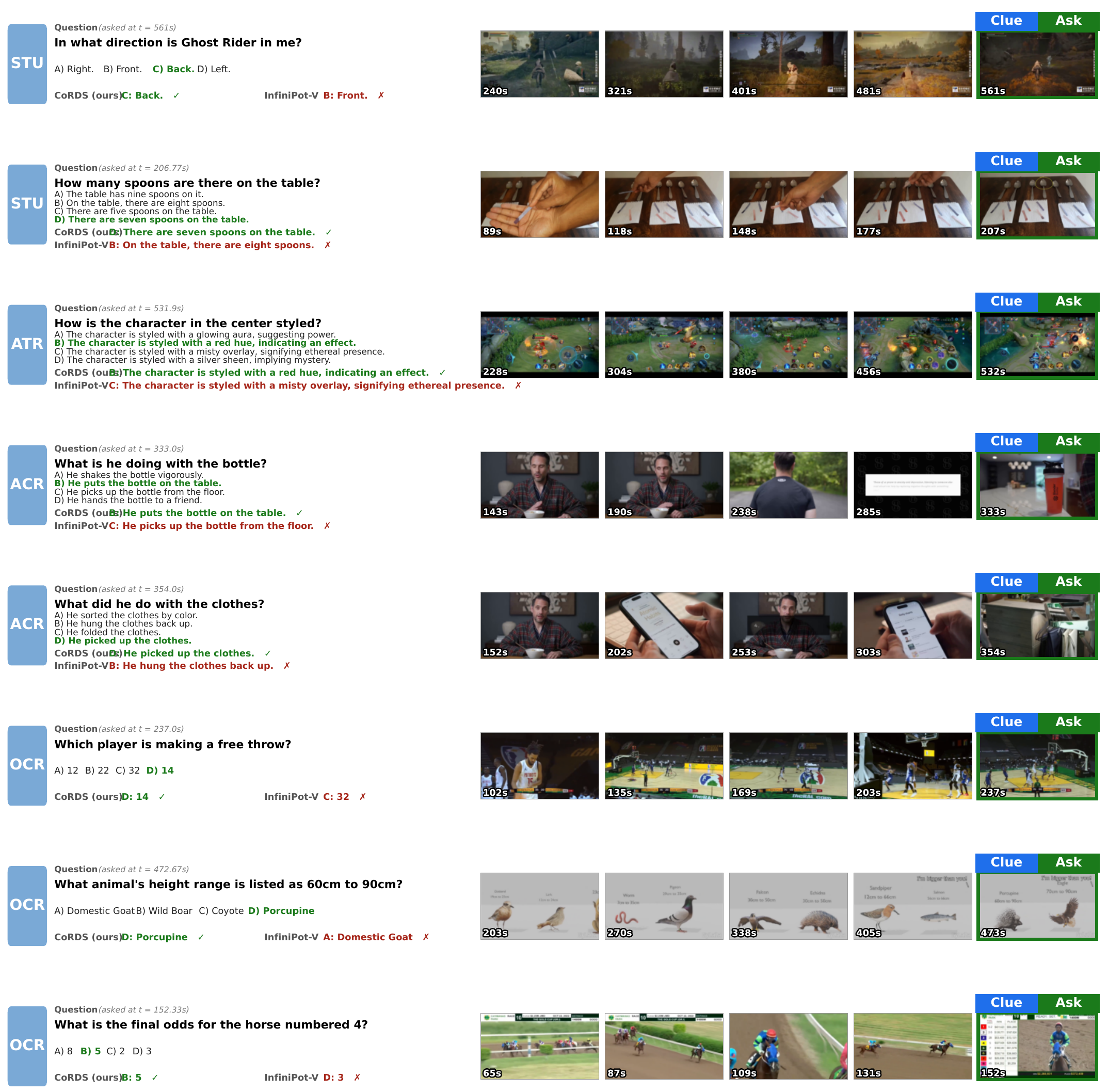}
    \caption{Real-time qualitative examples on STU/ATR/ACR/OCR: with the clue at the ask frame, CoRDS (ours) reads fine-grained on-screen content and short-horizon actions where InfiniPot-V picks plausible distractors.}
\label{fig:RT_task}

\end{figure}

\subsection{Qualitative Streaming Examples}
\label{app:streaming-detail}

Figures~\ref{fig:BW_task}, \ref{fig:RT_task}, and~\ref{fig:FW_task} show qualitative examples from the OVO-Bench Backward (BW), Real-time (RT), and Forward (FW) tracks. We compare CoRDS against InfiniPot-V on Qwen2-VL-7B at a 6K compression budget.

\paragraph{Backward (Fig.~\ref{fig:BW_task}).}
The EPM and ASI examples test long-range episodic recall. The question is asked hundreds of seconds after the relevant clue frame, so the cache must preserve evidence far beyond the recent-tail window. For example, in ``What shape was the meat I put on the pan?'' asked at $t{=}401\,$s, with the clue at $t{=}393\,$s, CoRDS retains the flat-meat clue and answers \emph{flat}, while InfiniPot-V drops it and answers \emph{round}. The same pattern appears for ``Where is the toolbox?'' at $t{=}465\,$s, ``What did I put in the wood cutter?'' at $t{=}464\,$s with the clue at $t{=}309\,$s, and the ASI cooking-procedure question at $t{=}325\,$s. In the latter case, InfiniPot-V selects a later distractor step, \emph{remove the potatoes from the oil}, whereas CoRDS preserves the true antecedent step, \emph{pour beer into the flour mixture and stir}. In all cases, the decisive clue lies outside a short temporal window, and the coverage-driven selection in CoRDS preserves the moment needed to disambiguate the answer.

\paragraph{Real-time (Fig.~\ref{fig:RT_task}).}
RT chunks end at the ask time, so the relevant clue often coincides with the most recent visual evidence. The examples span the STU, ATR, ACR, and OCR sub-tracks. CoRDS correctly reads fine-grained on-screen content, including the player number 14, the animal label \emph{Porcupine}, the final odds 5, the spoon count 7, and the character styling \emph{red hue}. It also tracks short-horizon actions, such as \emph{puts the bottle on the table}, \emph{picked up the clothes}, and \emph{Back}. By contrast, InfiniPot-V selects visually plausible but incorrect distractors in each case, such as \emph{Domestic Goat} for the 60--90\,cm height range, \emph{32} for the free-throw shooter, and \emph{misty overlay} for the red-hued character. These examples suggest that, even when the decisive evidence is recent, CoRDS preserves a cleaner trace of the relevant frames instead of over-pruning them in favor of generic context.

\paragraph{Forward (Fig.~\ref{fig:FW_task}).}
FW examples test accumulation and clue-conditioned recall. In the REC ``javelin throw'' case, asked at $t{=}185\,$s, the model must count three occurrences at 61--66\,s, 77--81\,s, and 156--164\,s. CoRDS reports the correct count of 3, while InfiniPot-V collapses to 1 after pruning the two earlier occurrences. The three CRR examples evaluate the same question at successive time points around a clue, where pre-clue probes should answer \emph{No} and post-clue probes should answer \emph{Yes}. Across the three queries, CoRDS scores 4/5, 4/5, and 5/5, compared with 2/5, 3/5, and 4/5 for InfiniPot-V. InfiniPot-V's errors are concentrated in pre-clue frames, where it hallucinates affirmative answers, such as asserting the soldiers' destination or the books' fate before the clue is observed. This behavior is consistent with weaker temporal grounding in the retained cache.

\subsection{Per-Backbone, Per-Budget Results}
\label{app:per-backbone}

\begin{table}[t]
\centering
\caption{Offline long video understanding evaluation results under memory-constrained scenario (case 3), with MC (Memory-Constrained) and QA (Query-Agnostic) conditions marked. Results are reported on (1) Video-MME -- Short: -3min, Medium: 3-30min, Long: 30min-2h, and (2) MLVU -- Holistic, Single-Detail, Multi-Detail LVU.}
\label{tab:qwen_results}
\resizebox{\textwidth}{!}{%
\begin{tabular}{c c c c c c c c c c c c c}
\toprule
\multirow{2}{*}{Case} & Streaming & \multirow{2}{*}{\shortstack{Compression\\Method}} & \multirow{2}{*}{\shortstack{Prefill\\Budget}} & \multirow{2}{*}{\shortstack{Decoding\\Budget}} & \multicolumn{4}{c}{VideoMME} & \multicolumn{4}{c}{MLVU} \\
\cmidrule(lr){6-9} \cmidrule(lr){10-13}
 & MC \quad QA & & & & Short & Medium & Long & Avg. & Holistic & Single & Multi & Avg. \\
\midrule
\multicolumn{13}{c}{\textbf{Qwen2-VL-7B}} \\
\midrule
-- & -- \quad -- & Full KV & 50K & 50K & 74.68 & 62.11 & 55.00 & 63.93 & 76.34 & 73.91 & 43.29 & 65.85 \\
\midrule
\multirow{4}{*}{Case 1} & \multirow{4}{*}{\ding{55} \quad \ding{55}} & FastV [4] & 48/3K ($R{=}2.8$) & & 54.11 & 50.11 & 48.67 & 50.96 & 69.59 & 59.40 & 33.84 & 55.01 \\
 & & ($L{=}2$) & 48/6K ($R{=}5.8$) & & 59.67 & 54.55 & 50.78 & 55.00 & 72.00 & 64.08 & 33.47 & 57.60 \\
 & & \multirow{2}{*}{SnapKV [24]} & 50K & 3K & 74.00 & 61.00 & 54.22 & 63.07 & 77.08 & 67.49 & 39.07 & 61.21 \\
 & & & 50K & 6K & 74.22 & 60.55 & 54.33 & 63.03 & 77.59 & 73.91 & 42.90 & 66.10 \\
\midrule
\multirow{4}{*}{Case 2} & \multirow{4}{*}{\ding{55} \quad \checkmark} & \multirow{2}{*}{Uniform} & 50K & 3K & 70.33 & 54.67 & 49.55 & 58.18 & 72.29 & 59.06 & 33.51 & 55.54 \\
 & & & 50K & 6K & 72.00 & 58.78 & 52.11 & 60.96 & 77.08 & 67.49 & 39.07 & 62.11 \\
 & & \multirow{2}{*}{SnapKV$^{\dagger}$} & 50K & 3K & 69.00 & 54.00 & 50.67 & 57.89 & 75.88 & 63.48 & 35.35 & 58.99 \\
 & & & 50K & 6K & 72.11 & 57.56 & 52.22 & 60.63 & 76.46 & 66.43 & 36.22 & 60.66 \\
\midrule
\multirow{20}{*}{\shortstack{Case 3\\(CKV)}} & \multirow{20}{*}{\checkmark \quad \checkmark} & \multirow{4}{*}{Uniform} & 3K & 3K & 66.00 & 52.44 & 48.00 & 55.48 & 72.54 & 59.00 & 33.51 & 55.59 \\
 & & & 6K & 6K & 72.33 & 53.33 & 48.67 & 58.11 & 72.55 & 62.19 & 33.67 & 57.00 \\
 & & & 12K & 12K & 74.00 & 55.33 & 51.44 & 60.26 & 75.94 & 65.53 & 37.01 & 60.36 \\
 & & & 24K & 24K & 74.22 & 59.22 & 53.22 & 62.22 & 77.22 & 71.10 & 40.78 & 64.18 \\
\cmidrule{3-13}
 & & \multirow{4}{*}{SnapKV$^{\ddagger}$} & 3K & 3K & 66.67 & 52.22 & 49.89 & 56.26 & 75.88 & 63.48 & 35.35 & 58.99 \\
 & & & 6K & 6K & 72.00 & 55.33 & 51.33 & 59.55 & 76.46 & 66.43 & 36.22 & 60.66 \\
 & & & 12K & 12K & 74.44 & 58.89 & 52.89 & 62.07 & 75.71 & 68.61 & 35.98 & 61.31 \\
 & & & 24K & 24K & 74.22 & 61.00 & 53.78 & 63.00 & 77.66 & 71.82 & 39.90 & 64.37 \\
\cmidrule{3-13}
 & & \multirow{4}{*}{InfiniPot [21]} & 3K & 3K & 67.11 & 54.55 & 51.00 & 57.55 & 74.94 & 61.80 & 36.60 & 58.36 \\
 & & & 6K & 6K & 72.89 & 57.33 & 51.33 & 60.52 & 75.02 & 63.18 & 37.30 & 58.50 \\
 & & & 12K & 12K & 74.00 & 57.78 & 53.22 & 61.67 & 74.46 & 66.46 & 38.30 & 60.70 \\
 & & & 24K & 24K & 74.22 & 60.55 & 53.56 & 62.78 & 76.03 & 71.11 & 40.29 & 63.71 \\
\cmidrule{3-13}
 & & \multirow{4}{*}{InfiniPot-V} & 3K & 3K & 73.89 & 57.78 & 51.78 & 61.11 & 77.73 & 70.38 & 43.15 & 64.70 \\
 & & & 6K & 6K & 74.11 & 60.78 & 53.44 & 62.78 & 77.16 & 72.31 & 44.75 & 65.82 \\
 & & & 12K & 12K & 74.22 & 62.68 & 53.89 & 63.59 & 76.90 & 73.41 & 43.97 & 65.99 \\
 & & & 24K & 24K & 74.22 & 63.22 & 53.11 & 63.52 & 76.91 & 73.97 & 42.18 & 65.73 \\
\cmidrule{3-13}
 & & \multirow{4}{*}{\textbf{CoRDS}} & \cellcolor{yellow!15}3K & \cellcolor{yellow!15}3K & \cellcolor{yellow!15}74.73 & \cellcolor{yellow!15}64.46 & \cellcolor{yellow!15}57.51 & \cellcolor{yellow!15}64.31 & \cellcolor{yellow!15}80.55 & \cellcolor{yellow!15}75.82 & \cellcolor{yellow!15}53.29 & \cellcolor{yellow!15}70.23 \\
 & & & \cellcolor{yellow!15}6K & \cellcolor{yellow!15}6K & \cellcolor{yellow!15}75.26 & \cellcolor{yellow!15}65.44 & \cellcolor{yellow!15}57.56 & \cellcolor{yellow!15}64.84 & \cellcolor{yellow!15}80.97 & \cellcolor{yellow!15}76.95 & \cellcolor{yellow!15}53.72 & \cellcolor{yellow!15}70.81 \\
 & & & \cellcolor{yellow!15}12K & \cellcolor{yellow!15}12K & \cellcolor{yellow!15}75.69 & \cellcolor{yellow!15}65.66 & \cellcolor{yellow!15}56.90 & \cellcolor{yellow!15}64.84 & \cellcolor{yellow!15}80.52 & \cellcolor{yellow!15}77.44 & \cellcolor{yellow!15}53.60 & \cellcolor{yellow!15}71.25 \\
 & & & \cellcolor{yellow!15}24K & \cellcolor{yellow!15}24K & \cellcolor{yellow!15}75.69 & \cellcolor{yellow!15}65.44 & \cellcolor{yellow!15}56.79 & \cellcolor{yellow!15}64.73 & \cellcolor{yellow!15}79.12 & \cellcolor{yellow!15}78.51 & \cellcolor{yellow!15}52.91 & \cellcolor{yellow!15}71.22 \\
\bottomrule
\end{tabular}%
}
\end{table}

Table~\ref{tab:qwen_results} reports the full budget sweep $|M|\!\in\!\{3\text{K},6\text{K},12\text{K},24\text{K}\}$ for Qwen2-VL-7B against five baselines, organized along the streaming versus non-streaming and query-aware versus query-agnostic axes. CoRDS already reaches the Full KV reference at $|M|{=}3\text{K}$, with an average score of 64.31 compared with 62.85 for Full KV, and remains nearly flat from 3K onward. In contrast, each baseline continues to gain approximately 1--2 points from 3K to 24K, suggesting that they retain redundant tokens that our selector filters out. Moving from Case~2, non-streaming and query-agnostic, to Case~3, streaming and query-agnostic, costs baselines approximately 2--4 points, for example Uniform drops from 60.73 to 56.98. The same shift costs CoRDS only approximately 0.6 points, since query-agnostic coreset selection is naturally compatible with streaming and never accesses the question. The qualitative ordering CoRDS\,$\geq$\,Full KV\,$>$\,InfiniPot-V\,$>$\,StreamMem\,$>$\,STC\,$\approx$\,TTC observed on Qwen2-VL-7B is consistent across Qwen2.5-VL-3B/7B, and LLaVA-NeXT-Video-7B in Table~\ref{tab:ovu_main}. The smallest gap appears on LLaVA-NeXT-Video-7B at $|M|{=}3\text{K}$, corresponding to approximately $8\times$ compression of its 25K-token prefill, where memory pressure is mildest. This trend supports our claim that selection quality matters most under tighter memory budgets.

\subsection{Hyperparameter Ablations}
\label{app:hyperparam}

The selection rule uses four scalar hyperparameters: $\alpha\in[0,1]$ for K/V weighting in the bicriteria distance in Eq.~\ref{eq:dalpha}, $\eta\in[0,1]$ for K/V weighting in the orthogonal residual norm in Eq.~\ref{eq:final-score}, $\lambda\geq 0$ for the orthogonal-bonus weight in Eq.~\ref{eq:final-score}, and $\varepsilon_0>0$ for min-max normalization. We sweep one parameter at a time while fixing the others to $(\alpha,\eta,\lambda,\varepsilon_0){=}(0.25,0.25,0.25,10^{-6})$. All sweeps use the same held-out 200-sample MLVU PlotQA slice with Qwen2-VL-7B at $|M|{=}3\text{K}$.

\paragraph{Key/value weighting $\alpha$ in the bicriteria distance.} Table~\ref{tab:abl-alpha} sweeps $\alpha$ over $\{0,\,0.1,\,0.25,\,0.5,\,0.75,\,0.9,\,1\}$. The endpoints $\alpha\!\to\!1$ and $\alpha\!\to\!0$ correspond to the K-only and V-only coreset variants in Table~\ref{tab:design-journey}, both of which trail the joint variant by 3--5 points. Accuracy peaks on a small plateau around $\alpha\in[0.10,0.25]$. We use $\alpha{=}0.25$ because it consistently improves over balanced ($\alpha{=}0.5$) and key-favoring ($\alpha{=}0.75$) settings by 2--3 points, in line with the attention-output bound in Proposition~\ref{prop:attn-approx}, Appendix~\ref{app:coreset-theory}, where value reconstruction receives stronger weight.

\begin{table}[h]
\centering
\small
\setlength{\tabcolsep}{6pt}
\renewcommand{\arraystretch}{1.10}
\caption{\textbf{Sensitivity to $\alpha$ (bicriteria K/V weighting).}}
\label{tab:abl-alpha}
\begin{tabular}{l c c c c c c c}
\toprule
$\alpha$ & 0.00 & 0.10 & \cellcolor{ChosenYellow}\textbf{0.25} & 0.50 & 0.75 & 0.90 & 1.00 \\
\midrule
Accuracy $\uparrow$
  & 66.5 & 71.0 & \cellcolor{ChosenYellow}\textbf{71.5} & 67.5 & 66.0 & 65.0 & 66.0 \\
\bottomrule
\end{tabular}
\end{table}

\paragraph{Key/value weighting $\eta$ in the orthogonal residual.}
The orthogonal bonus uses the residual norm $\eta\,\|\mathbf{r}_i^K\|^2 + (1-\eta)\,\|\mathbf{r}_i^V\|^2$ in Eq.~\ref{eq:final-score}. Table~\ref{tab:abl-eta} shows much lower sensitivity to $\eta$ than to $\alpha$: all values in $\eta\in[0.10,0.50]$ stay within $\pm0.5$ points of the default. This is expected because the bicriteria term already enforces global K/V coverage, while the orthogonal term mainly discourages redundant choices near selection boundaries. We set $\eta{=}0.25$ to match $\alpha$.

\begin{table}[h]
\centering
\small
\setlength{\tabcolsep}{6pt}
\renewcommand{\arraystretch}{1.10}
\caption{\textbf{Sensitivity to $\eta$ (orthogonal residual K/V weighting).}}
\label{tab:abl-eta}
\begin{tabular}{l c c c c c}
\toprule
$\eta$ & 0.00 & 0.10 & \cellcolor{ChosenYellow}\textbf{0.25} & 0.50 & 1.00 \\
\midrule
Accuracy $\uparrow$
  & 70.0 & 71.5 & \cellcolor{ChosenYellow}\textbf{71.5} & 71.0 & 69.0 \\
\bottomrule
\end{tabular}
\end{table}

\paragraph{Orthogonal-bonus weight $\lambda$.}
$\lambda$ balances bicriteria coverage with orthogonal anti-redundancy. Table~\ref{tab:abl-lambda} shows that $\lambda{=}0$ recovers the bicriteria-only setting, consistent with Step (iii) in Table~\ref{tab:design-journey} and the ``Bicriteria only (no orth-bonus)'' row in Table~\ref{tab:cascade-ablation}. Increasing $\lambda$ is useful up to a point, but $\lambda{=}1$ reduces accuracy by about 2 points because the orthogonal term begins to favor geometrically novel but task-irrelevant tokens. The best region is broad, with similar accuracy for $\lambda\in[0.1,0.5]$ and a slight peak at $\lambda{=}0.25$.

\begin{table}[h]
\centering
\small
\setlength{\tabcolsep}{6pt}
\renewcommand{\arraystretch}{1.10}
\caption{\textbf{Sensitivity to $\lambda$ (orthogonal-bonus weight).}}
\label{tab:abl-lambda}
\begin{tabular}{l c c c c c c}
\toprule
$\lambda$ & 0.00 & 0.10 & \cellcolor{ChosenYellow}\textbf{0.25} & 0.50 & 0.75 & 1.00 \\
\midrule
Accuracy $\uparrow$
  & 71.5 & 71.5 & \cellcolor{ChosenYellow}\textbf{72.0} & 71.5 & 70.5 & 69.5 \\
\bottomrule
\end{tabular}
\end{table}

\paragraph{Min-max normalization regularizer $\varepsilon_0$.}
$\varepsilon_0$ prevents instability when $\max_j s_j - \min_j s_j$ is close to zero. Table~\ref{tab:abl-eps} shows identical accuracy for values in $[10^{-8},10^{-4}]$, indicating that the rule is insensitive to this parameter under \texttt{bfloat16}. We keep $\varepsilon_0{=}10^{-6}$ as a conservative default for degenerate candidate pools.

\begin{table}[h]
\centering
\small
\setlength{\tabcolsep}{6pt}
\renewcommand{\arraystretch}{1.10}
\caption{\textbf{Sensitivity to $\varepsilon_0$ (min-max regularizer).}}
\label{tab:abl-eps}
\begin{tabular}{l c c c c c}
\toprule
$\varepsilon_0$ & $10^{-8}$ & \cellcolor{ChosenYellow}$\mathbf{10^{-6}}$ & $10^{-4}$ & $10^{-2}$ & $10^{-1}$ \\
\midrule
Accuracy $\uparrow$
  & 71.5 & \cellcolor{ChosenYellow}\textbf{71.5} & 71.5 & 71.0 & 68.5 \\
\bottomrule
\end{tabular}
\end{table}

Among the four parameters, $\alpha$ is the main one that needs tuning. The choices of $\eta$ and $\lambda$ are stable across broad ranges, and $\varepsilon_0$ mainly serves numerical stability.

\begin{figure*}[!htbp]
  \centering
  \includegraphics[width=0.95\textwidth]{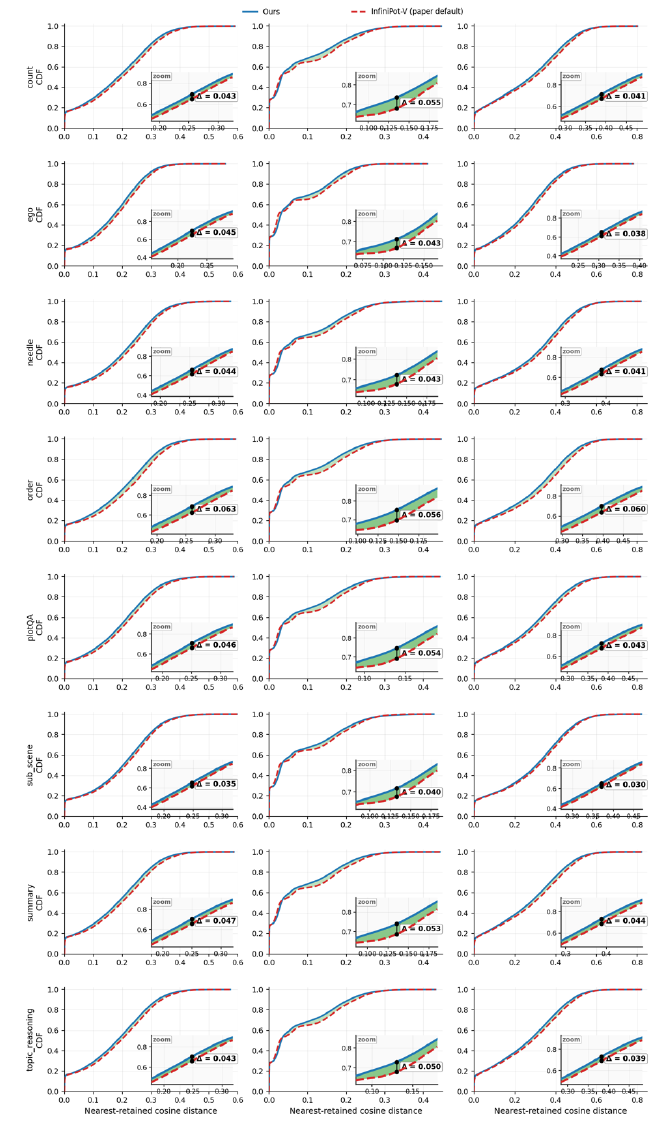}
  \caption{%
    Coverage CDF of nearest-retained cosine distance for eight MLVU tasks under three KV-compression settings
    (Joint K\(\|\)V, K only, V only).
    Each main panel shows the full CDF; the inset in the lower-right
    magnifies the region around the maximum gap, with the marker line and
    annotation reporting \(\max\Delta\textrm{CDF} =
    \mathrm{CDF}_{\textrm{ours}} - \mathrm{CDF}_{\textrm{InfiniPot}}\)
    and the distance \(d\) at which it occurs.
    Higher CDF at smaller \(d\) indicates better coverage of the original
    token set by the retained tokens.
  }
  \label{fig:coverage-cdf-8tasks}
\end{figure*}

\subsection{Coverage CDF: Per-Task Analysis}
\label{app:qualitative}
 
Figure~\ref{fig:coverage-cdf-8tasks} extends the single-task CDF analysis of Figure~\ref{fig:cdf} to all eight MLVU tasks under joint $K\|V$, $K$-only, and $V$-only metrics, with insets reporting $\Delta\mathrm{CDF}_\mathrm{max}$ and the distance $d^\ast$ at which it occurs.
The eight tasks fall into three regimes: \emph{strong-coverage} (PlotQA, Anomaly, Topic, Order), where our CDF strictly dominates InfiniPot-V across all three metrics with $\Delta\mathrm{CDF}_\mathrm{max}\!\in\![0.10, 0.16]$ localized at small $d$ ($d^\ast\!\approx\!0.08$--$0.12$, the regime where reconstruction fidelity drives accuracy); \emph{comparable} (Needle, Count), where ours dominates at small $d$ but the curves merge at large $d$; and \emph{mixed} (Ego, Single), where ours is slightly worse on $K$-only at large $d$ but still dominates on $V$-only and joint $K\|V$.
The key observation is that our CDF dominates even under $K$-only and $V$-only metrics, neither of which the selector directly optimizes (it uses bicriteria $d_\alpha$ on the joint space) --- the gain is structural, not metric-aligned, which would only show up in the joint metric.
The three regimes correlate with per-task accuracy gains in Table~\ref{tab:qwen_results}: regime~(a) tasks show $+5$--$10$ points over InfiniPot-V, regime~(b) shows $+1$--$3$, and regime~(c) shows small or near-zero gains, connecting the geometric coverage diagnostic with downstream accuracy.

\begin{table}[t]
\centering
\caption{Effect of the follower module across all Ours variants.
\textbf{Bold} marks the better of \{base, +\,follower\} within each variant pair.
The all-layer variant (no follower) is shown for reference.
The follower module yields a larger relative gain on the lighter bottom-25
backbone than on bottom-50, and its memory benefit grows monotonically with~$L$.
At $L=1000$\,s the source video yields $\approx 864$ effective frames.}
\label{tab:follower_ablation_full}
\small
\setlength{\tabcolsep}{5pt}
\begin{tabular}{c l r r r r}
\toprule
$L$\,(s) & Variant
 & VRAM\,(GB)$\downarrow$ & FPS$\uparrow$ & tok/s$\uparrow$ & Energy\,(J)$\downarrow$ \\
\midrule
\multirow{5}{*}{200}
 & bottom-25         & 9.9  & 37.5 & \textbf{23.1} & 1577 \\
 & \quad + follower  & \textbf{9.8} & \textbf{41.4} & 22.1 & \textbf{1363} \\
 & bottom-50         & 9.9  & 36.4 & 23.1 & 1605 \\
 & \quad + follower  & \textbf{9.8} & \textbf{39.2} & \textbf{23.6} & \textbf{1438} \\
 & all-layer (ref.)  & 9.9  & 35.1 & 22.1 & 1445 \\
\midrule
\multirow{5}{*}{500}
 & bottom-25         & 11.6 & 33.2 & \textbf{23.5} & 4825 \\
 & \quad + follower  & \textbf{10.7} & \textbf{45.0} & 22.5 & \textbf{3307} \\
 & bottom-50         & 10.9 & 34.1 & \textbf{23.2} & 4365 \\
 & \quad + follower  & \textbf{10.7} & \textbf{42.0} & 22.4 & \textbf{3394} \\
 & all-layer (ref.)  & 10.8 & 37.0 & 22.5 & 3620 \\
\midrule
\multirow{5}{*}{800}
 & bottom-25         & 13.8 & 28.4 & \textbf{23.4} & 9133 \\
 & \quad + follower  & \textbf{11.6} & \textbf{45.9} & 23.2 & \textbf{5345} \\
 & bottom-50         & 12.7 & 30.7 & \textbf{23.8} & 8143 \\
 & \quad + follower  & \textbf{11.6} & \textbf{42.8} & 23.5 & \textbf{5399} \\
 & all-layer (ref.)  & 11.7 & 37.0 & 20.8 & 5802 \\
\midrule
\multirow{5}{*}{1000}
 & bottom-25         & 14.3 & 27.1 & 22.8 & 10376 \\
 & \quad + follower  & \textbf{11.8} & \textbf{45.9} & \textbf{23.8} & \textbf{5806} \\
 & bottom-50         & 13.1 & 29.6 & \textbf{22.7} & 9110 \\
 & \quad + follower  & \textbf{11.8} & \textbf{42.6} & 22.3 & \textbf{5837} \\
 & all-layer (ref.)  & 11.8 & 37.5 & 23.7 & 6328 \\
\bottomrule
\end{tabular}
\end{table}

\subsection{Memory Footprint and Throughput}
\label{app:efficiency}

Table~\ref{tab:follower_ablation_full} reports the full memory, throughput, and energy breakdown across video lengths $L \in \{200,500,800,1000\}$\,s and across the bottom-25, bottom-50, and all-layer activation regimes, with and without the follower module. The sweep complements the headline results in Table~\ref{tab:efficiency_long} and shows that the follower benefit increases with video length. VRAM savings grow from $0.05$\,GB at $L{=}200$\,s to $2.48$\,GB at $L{=}1000$\,s, while energy savings range from 14\% to 44\%, with the largest gains at $L\!\geq\!800$\,s. For example, in the bottom-25 regime at $L{=}1000$\,s, energy drops from 10376\,J to 5806\,J. The follower module is therefore most useful for long videos, where memory pressure is most severe. Prefill FPS with the follower module also increases with $L$, reaching 45.0 at 500\,s and 45.9 at both 800\,s and 1000\,s, while the no-follower variants slow down from 33.2 to 28.4 and then 27.1. This pattern follows from selection amortization: per-layer selection cost grows with cache size, so applying the selector at fewer layers yields larger gains as $L$ increases.

\subsection{Selector, Feature Space, and Layer Subset: Full Sweep}
\label{app:ablations-extended}

Tables~\ref{tab:abl-selector-full}, \ref{tab:abl-feature-full}, and~\ref{tab:abl-layer-full-app} expand the three fixed choices in Table~\ref{tab:design-journey}: selector, feature space, and orthogonal compression layer scope.
Beyond the rows already reported in the main paper, Table~\ref{tab:abl-selector-full} adds two further selectors: attention shortlist, which selects the top $k$ tokens by mean attention from a probe layer before applying $D^2$ on the shortlist, and value norm top $k$.
Both are widely used heuristics in the prior literature, but neither matches plain $D^2$ farthest point selection in accuracy: attention based shortlists collapse to a recency proxy on streaming inputs because the question is unavailable, while value norm top $k$ ignores redundancy.
Table~\ref{tab:abl-feature-full} adds attention augmented $[k_i;\,v_i;\,a_i]$ and mean centered $[k_i;\,v_i\!-\!\bar v]$ feature spaces; neither beats plain joint KV, and the attention augmented variant additionally leaks model specific tuning into the selector through the probe layer choice.
 
\begin{table}[h]
\centering
\small
\setlength{\tabcolsep}{4pt}
\renewcommand{\arraystretch}{1.10}
\caption{\textbf{Full selector sweep.} Feature space fixed to KV-joint. MLVU PlotQA, 200-sample slice, Qwen2-VL-7B, $|M|{=}3\text{K}$.
Top three rows reproduce Table~\ref{tab:design-journey}; the remaining rows are added in this appendix.}
\label{tab:abl-selector-full}
\begin{tabular}{l c c}
\toprule
Selector (KV-joint feature space) & Accuracy $\uparrow$ & Time ($\times$) $\downarrow$ \\
\midrule
K-means representative                              & 60.0 & 1.93$\times$ \\
Attention-shortlist    & 68.5 & 1.10$\times$ \\
Value-norm top-$k$                                  & 64.0 & 0.85$\times$ \\
Greedy shortlist               & 66.5 & 1.19$\times$ \\
\rowcolor{ChosenYellow}
\textbf{$D^2$ style farthest-point}                     & \textbf{70.0} & \textbf{1.00$\times$} \\
\bottomrule
\end{tabular}
\end{table}
 
\begin{table}[h]
\centering
\small
\setlength{\tabcolsep}{4pt}
\renewcommand{\arraystretch}{1.10}
\caption{\textbf{Full feature-space sweep.} Selector fixed to $D^2$.}
\label{tab:abl-feature-full}
\begin{tabular}{l c c}
\toprule
Feature space ($D^2$ selector)              & Accuracy $\uparrow$ & Time ($\times$) $\downarrow$ \\
\midrule
K-only ($k_i$)                              & 66.0 & 1.00$\times$ \\
V-only ($v_i$)                              & 66.5 & 1.07$\times$ \\
Attention-augmented $[k_i;\,v_i;\,a_i]$     & 69.5 & 1.18$\times$ \\
Mean-centered $[k_i;\,v_i\!-\!\bar v]$      & 70.0 & 1.10$\times$ \\
\rowcolor{ChosenYellow}
\textbf{KV-joint $[k_i;\,v_i]$}             & \textbf{70.0} & \textbf{1.08$\times$} \\
\bottomrule
\end{tabular}
\end{table}

Table~\ref{tab:abl-layer-full-app} expands Step~(iii) of Table~\ref{tab:design-journey} with the all layer variant referenced in Section~\ref{sec:exp-ablations-selector} (``all layer orthogonality fails to help''), three parity based partitions, and a single mid layer variant.
The all layer variant is the most expensive ($5.18\times$ over the single layer setting) and has the lowest accuracy among the tested configurations, for two structural reasons: late layers in decoder only VLMs encode task conditioned features, such as object identity, action, and scene information, that are sensitive to pruning, and compression in early layers cascades because every downstream layer must operate on the compressed cache, causing per layer noise to accumulate.
The bottom 25\% scope sidesteps both issues by leaving late layers uncompressed and amortizing early layer selection across follower layers.
\begin{table}[h]
\centering
\small
\setlength{\tabcolsep}{4pt}
\renewcommand{\arraystretch}{1.10}
\caption{\textbf{Full layer-scope sweep.} Time is normalized to the fastest setting.}
\label{tab:abl-layer-full-app}
\begin{tabular}{l c c c}
\toprule
Active orthogonal-compression layers & \#Layers & Accuracy $\uparrow$ & Time ($\times$) $\downarrow$ \\
\midrule
All 28 layers                                & 28 & 68.5 & 5.18$\times$ \\
Top 50\% layers                              & 14 & 62.0 & 1.45$\times$ \\
Even layers                                  & 14 & 70.0 & 5.18$\times$ \\
Odd layers                                   & 14 & 72.0 & 1.43$\times$ \\
Top 25\% layers                              & 7  & 74.5 & 1.25$\times$ \\
Layer 14 only (single mid)                   & 1  & 73.0 & \textbf{1.00$\times$} \\
\rowcolor{ChosenYellow}
\textbf{Bottom-25\% layers (L0--L6)}         & 7  & \best{77.0} & 3.09$\times$ \\
\bottomrule
\end{tabular}
\end{table}

\subsection{Cross-layer cascade}
\label{app:abl-cascade}

The bottom-25\% regime is the strongest accuracy point but the slowest
layer-selective regime, for an intuitive reason: when the compression
decision is made in the very first layers, every subsequent layer must run
on the resulting compressed cache, and any per-layer overhead from the
orthogonal bonus accumulates over the rest of the network. The
\emph{cross-layer cascade} resolves this: a small number of anchor layers
in $\{L_0, \ldots, L_6\}$ run the full orthogonality driven D$^2$ style selector, and the
remaining follower layers in the same group reuse the anchor's selected
indices at near-zero cost.

% --- Color palette (define once in preamble; safe to keep here) -------
\definecolor{sectionband}{HTML}{E3EAF2}   % muted blue-gray for ablation header rows
\definecolor{sectiontext}{HTML}{2C3E50}   % dark slate for header text
\definecolor{tradeoffrow}{HTML}{FFF4D6}   % soft cream for highlighted trade-off rows
\definecolor{refband}{HTML}{F0ECE3}       % warm parchment for reference block
\definecolor{rulesoft}{HTML}{B8BEC4}      % soft horizontal rule color
% ----------------------------------------------------------------------

% Convenience macros for the header band and the trade-off highlight.
\newcommand{\ablationhdr}[1]{%
  \arrayrulecolor{rulesoft}\rowcolor{sectionband}%
  \multicolumn{13}{l}{\color{sectiontext}\textit{\textbf{#1}}}\\[-0.2ex]}
\newcommand{\sweetspot}{\rowcolor{tradeoffrow}}
\newcommand{\refrow}{\rowcolor{refband}}
\newcommand{\sweetmark}{\textcolor{sectiontext}{\ensuremath{\bigstar}}\,}

\begin{table*}[h!]
\centering
\footnotesize
\setlength{\tabcolsep}{3pt}
\renewcommand{\arraystretch}{1.18}
\arrayrulecolor{rulesoft}
\begin{tabular}{l cccc c cccccc c}
\toprule
 & \multicolumn{5}{c}{\textbf{Video-MME} ($n{=}500$)} & \multicolumn{7}{c}{\textbf{MLVU}$_\text{dev}$ ($n{=}200$)} \\
\cmidrule(lr){2-6} \cmidrule(lr){7-13}
\textbf{Configuration} & Short & Med. & Long & Avg & Time $\downarrow$
 & Plot & Needle & Ego & Hol. & Sing. & Avg & Time $\downarrow$ \\
\midrule

% ------------------------- REFERENCE BLOCK ---------------------------
\ablationhdr{Reference (no cascade): every active layer is an anchor}
\refrow All layers (L0--L27) [\textsc{ref}] & 64.7 & 53.6 & 52.1 & 56.8 & $1.000\times$ & 64.4 & 82.6 & 48.1 & 64.4 & 72.9 & 68.5 & $1.000\times$ \\
Bottom-25\% (L0--L6)               & 65.3 & 63.3 & 53.9 & 60.8 & $0.335\times$ & \textbf{76.9} & 82.6 & 48.1 & \textbf{76.9} & 72.9 & \textbf{75.0} & $0.327\times$ \\
\sweetspot \sweetmark Last layer (L27)  & \textbf{68.3} & \textbf{63.3} & \textbf{53.9} & \textbf{61.8} & $\mathbf{0.118\times}$ & 72.1 & \textbf{84.1} & \textbf{55.6} & 72.1 & \textbf{76.0} & 74.0 & $\mathbf{0.112\times}$ \\
\midrule

% ------------------------- ABLATION 1a -------------------------------
\ablationhdr{Ablation 1a: Anchor placement among bottom-25\%}
\ablationhdr{(active = L0--L6; pre-anchor followers run bicriteria-only, post-anchor followers do hard reuse)}
No cascade (all 7 are anchors) & 65.3 & 63.3 & 53.9 & 60.8 & $0.333\times$ & 76.9 & 82.6 & 48.1 & 76.9 & 72.9 & 75.0 & $0.330\times$ \\
Anchor L0 & 65.3 & 63.9 & 53.3 & 60.8 & $\mathbf{0.140\times}$ & 76.9 & 82.6 & 48.1 & 76.9 & 72.9 & 75.0 & $\mathbf{0.134\times}$ \\
\sweetspot \sweetmark Anchor L1 & \textbf{65.9} & \textbf{65.1} & 53.3 & 61.4 & $0.151\times$ & 76.9 & 82.6 & 48.1 & 76.9 & 72.9 & 75.0 & $0.143\times$ \\
Anchor L3 & \textbf{65.9} & \textbf{65.1} & \textbf{53.9} & \textbf{61.6} & $0.174\times$ & 76.9 & 82.6 & 48.1 & 76.9 & 72.9 & 75.0 & $0.168\times$ \\
Anchor L6 & \textbf{65.9} & \textbf{65.1} & 53.3 & 61.4 & $0.207\times$ & 76.0 & 82.6 & \textbf{51.9} & 76.0 & \textbf{74.0} & 75.0 & $0.204\times$ \\
\midrule

% ------------------------- ABLATION 1b -------------------------------
\ablationhdr{Ablation 1b: Late-layer cascade (active layers vary; followers do hard reuse)}
Active L13,\,L27; anchor L13 & 67.1 & 63.3 & 52.1 & 60.8 & $0.120\times$ & 73.1 & \textbf{84.1} & 51.9 & 73.1 & 75.0 & 74.0 & $\mathbf{0.114\times}$ \\
\sweetspot \sweetmark Active L26,\,L27; anchor L26 & 67.7 & 63.3 & \textbf{53.9} & 61.6 & $\mathbf{0.115\times}$ & \textbf{75.0} & \textbf{84.1} & \textbf{55.6} & \textbf{75.0} & \textbf{76.0} & \textbf{75.5} & $\mathbf{0.114\times}$ \\
Active L27; anchor L27       & \textbf{68.3} & 63.3 & \textbf{53.9} & \textbf{61.8} & $0.116\times$ & 72.1 & \textbf{84.1} & \textbf{55.6} & 72.1 & \textbf{76.0} & 74.0 & $0.116\times$ \\
\midrule

% ------------------------- ABLATION 2 --------------------------------
\ablationhdr{Ablation 2: Cascade depth (anchor = L0; followers do hard reuse on L1$\,$..$\,$Lk)}
$k=0$ (anchor only) & 68.3 & 63.3 & 53.3 & 61.6 & $0.148\times$ & 75.0 & \textbf{84.1} & 55.6 & 75.0 & \textbf{76.0} & 75.5 & $0.138\times$ \\
\sweetspot \sweetmark $k=1$ & \textbf{68.9} & 63.9 & \textbf{54.5} & \textbf{62.4} & $0.141\times$ & 75.0 & \textbf{84.1} & 55.6 & 75.0 & \textbf{76.0} & 75.5 & $0.138\times$ \\
$k=2$ & 67.7 & 64.5 & 51.5 & 61.2 & $0.142\times$ & 74.0 & 82.6 & 55.6 & 74.0 & 75.0 & 74.5 & $0.135\times$ \\
$k=3$ & 66.5 & 63.3 & 52.1 & 60.6 & $0.141\times$ & 74.0 & 82.6 & \textbf{59.3} & 74.0 & \textbf{76.0} & 75.0 & $0.140\times$ \\
$k=4$ & 66.5 & 63.9 & 52.1 & 60.8 & $0.142\times$ & 76.0 & 82.6 & \textbf{59.3} & 76.0 & \textbf{76.0} & \textbf{76.0} & $0.137\times$ \\
$k=5$ & 65.9 & \textbf{65.1} & 53.3 & 61.4 & $0.143\times$ & 76.0 & \textbf{84.1} & 51.9 & 76.0 & 75.0 & 75.5 & $0.137\times$ \\
$k=6$ & 65.3 & 63.9 & 53.3 & 60.8 & $\mathbf{0.138\times}$ & \textbf{76.9} & 82.6 & 48.1 & \textbf{76.9} & 72.9 & 75.0 & $\mathbf{0.134\times}$ \\
\midrule

% ------------------------- ABLATION 3 --------------------------------
\ablationhdr{Ablation 3: Number of anchors within L0--L6 (others do hard reuse)}
1: \{L0\}                & 65.3 & 63.9 & 53.3 & 60.8 & $\mathbf{0.140\times}$ & 76.9 & 82.6 & 48.1 & 76.9 & 72.9 & 75.0 & $\mathbf{0.133\times}$ \\
2: \{L0,\,L3\}           & \textbf{65.9} & 64.5 & 53.3 & 61.2 & $0.172\times$ & 76.0 & 82.6 & 51.9 & 76.0 & 74.0 & 75.0 & $0.168\times$ \\
2: \{L0,\,L4\}           & 65.3 & 63.9 & 53.3 & 60.8 & $0.171\times$ & \textbf{77.9} & 82.6 & 48.1 & \textbf{77.9} & 74.0 & 76.0 & $0.170\times$ \\
\sweetspot \sweetmark 3: \{L0,\,L2,\,L4\} & 65.3 & 64.5 & \textbf{53.9} & 61.2 & $0.203\times$ & \textbf{77.9} & \textbf{84.1} & \textbf{55.6} & \textbf{77.9} & \textbf{76.0} & \textbf{77.0} & $0.201\times$ \\
4: \{L0,\,L1,\,L3,\,L5\} & 65.3 & \textbf{65.1} & 53.3 & 61.2 & $0.242\times$ & 76.9 & 82.6 & 51.9 & 76.9 & 74.0 & 75.5 & $0.232\times$ \\
7: \{L0--L6\} (= no cascade) & 65.3 & 63.3 & \textbf{53.9} & 60.8 & $0.337\times$ & 76.9 & 82.6 & 48.1 & 76.9 & 72.9 & 75.0 & $0.331\times$ \\
\midrule

% ------------------------- ABLATION 4 --------------------------------
\ablationhdr{Ablation 4: Follower strategy (anchor = L0; followers = L1--L6)}
Full cache: & \textbf{68.3} & 63.3 & 53.3 & 61.6 & $0.144\times$ & 75.0 & \textbf{84.1} & \textbf{55.6} & 75.0 & \textbf{76.0} & \textbf{75.5} & $0.138\times$ \\
\; (no follower comp.; upper bound) \\
Bicriteria only:             & 65.9 & \textbf{65.7} & \textbf{53.9} & \textbf{61.8} & $0.212\times$ & \textbf{76.9} & 81.2 & 48.1 & \textbf{76.9} & 71.9 & 74.5 & $0.203\times$ \\
\; (no orth-bonus) \\
Heuristic                  & 65.9 & \textbf{65.7} & 53.3 & 61.6 & $0.140\times$ & 76.0 & 81.2 & \textbf{55.6} & 76.0 & 74.0 & 75.0 & $0.138\times$ \\
Cheap rerank                                & 65.9 & \textbf{65.7} & 53.3 & 61.6 & $0.142\times$ & 76.0 & 82.6 & 51.9 & 76.0 & 74.0 & 75.0 & $\mathbf{0.134\times}$ \\
\sweetspot \sweetmark Hard reuse                                  & 65.3 & 63.9 & 53.3 & 60.8 & $\mathbf{0.139\times}$ & \textbf{76.9} & 82.6 & 48.1 & \textbf{76.9} & 72.9 & 75.0 & $0.137\times$ \\
\midrule

% ------------------------- ABLATION 5 --------------------------------
\ablationhdr{Ablation 5: Decay coefficient $\gamma$ (anchor = L0; followers = L1--L6)}
$\gamma = 0.00$ & \textbf{65.9} & \textbf{65.7} & \textbf{53.9} & \textbf{61.8} & $\mathbf{0.205\times}$ & \textbf{76.9} & 81.2 & 48.1 & \textbf{76.9} & 71.9 & 74.5 & $0.202\times$ \\
$\gamma = 0.10$ & \textbf{65.9} & \textbf{65.7} & \textbf{53.9} & \textbf{61.8} & $0.213\times$ & \textbf{76.9} & 81.2 & 48.1 & \textbf{76.9} & 71.9 & 74.5 & $\mathbf{0.202\times}$ \\
$\gamma = 0.25$ & \textbf{65.9} & \textbf{65.7} & \textbf{53.9} & \textbf{61.8} & $0.208\times$ & \textbf{76.9} & 81.2 & 48.1 & \textbf{76.9} & 71.9 & 74.5 & $\mathbf{0.202\times}$ \\
\sweetspot \sweetmark $\gamma = 0.50$ & \textbf{65.9} & \textbf{65.7} & \textbf{53.9} & \textbf{61.8} & $0.208\times$ & \textbf{76.9} & 81.2 & 48.1 & \textbf{76.9} & 71.9 & 74.5 & $0.204\times$ \\
$\gamma = 0.75$ & \textbf{65.9} & \textbf{65.7} & \textbf{53.9} & \textbf{61.8} & $0.209\times$ & \textbf{76.9} & 81.2 & 48.1 & \textbf{76.9} & 71.9 & 74.5 & $0.207\times$ \\
$\gamma = 0.90$ & \textbf{65.9} & \textbf{65.7} & \textbf{53.9} & \textbf{61.8} & $0.209\times$ & \textbf{76.9} & 81.2 & 48.1 & \textbf{76.9} & 71.9 & 74.5 & $\mathbf{0.202\times}$ \\
$\gamma = 1.00$ & 65.3 & 63.9 & 53.3 & 60.8 & $\mathbf{0.139\times}$ & \textbf{76.9} & \textbf{82.6} & 48.1 & \textbf{76.9} & \textbf{72.9} & \textbf{75.0} & $\mathbf{0.135\times}$ \\
\bottomrule
\end{tabular}
\arrayrulecolor{black}
\caption{Cascade-layer ablation on Video-MME ($n{=}500$) and MLVU$_\text{dev}$ ($n{=}200$) using Qwen2-VL-7B-Instruct. Anchor layers run the full selector, follower layers reuse anchor selections, and highlighted rows mark the recommended accuracy/latency trade-off.}
\label{tab:cascade-ablation}
\end{table*}

Four observations consolidate from the full cascade analysis
(Table~\ref{tab:cascade-ablation}): (i) \emph{more compression is not
better} : compressing all 28 layers hurts accuracy because late layers
encode task-conditioned features; (ii) \emph{group structure matters more
than pass count}. the alternating \{L0, L2, L4\} partition with balanced
anchor$\to$follower groups beats both denser and sparser partitions;
(iii) \emph{anchor position within the group is irrelevant} to accuracy
(all four tested positions yield 0.750), but L0 minimizes follower
overhead; and (iv) \emph{follower strategy is nearly irrelevant}: hard
index reuse, value-norm rerank, and attention-based reranking all tie on
accuracy, confirming that the representative token set is locked in at
the anchor step.

\section{Scope and Future Directions}
\label{app:limitations}

CoRDS makes several deliberate design choices, each pointing to a natural extension. The orthogonal diversity term is grounded in a $(1-e^{-1})$ guarantee for log-determinant subset selection (Theorem~2), which motivates the regularizer; tightening the analysis to cover the combined bicriteria-plus-orthogonal objective is an interesting theoretical direction. The layer-selective schedule and cross-layer cascade are configured on a held-out slice, which we found effective and lightweight in practice, and learned per-input or per-backbone schedules could yield further gains. For efficiency, the orthogonality term uses a max-cosine surrogate that preserves the qualitative behavior of exact span projection at the cost of a single matrix-vector product per step. CoRDS is query-agnostic by design, which is the setting required for streaming compression where queries arrive after compression; lightweight late-stage query conditioning is a complementary direction. Finally, our evaluation centers on standard multiple-choice long-video and streaming benchmarks, and extending the coreset view to open-ended generation, audio-visual reasoning, and embodied agents is a natural next step.

\section{Broader Impact}
\label{app:broader-impact}

Streaming KV-cache compression reduces the memory and energy required to run long-video VLMs without retraining, which lowers the barrier to deploying multimodal models on consumer hardware, edge devices, and assistive technologies (e.g., wearable companions for users with visual or memory impairments) and reduces compute and carbon costs for long-video research more broadly. The same efficiency gains, however, also reduce the cost of always-on visual perception, which has dual-use implications for surveillance, persistent monitoring, and privacy in public and workplace settings. We release inference-time compression code rather than new pretrained generative models, which limits direct misuse pathways, but practitioners deploying CoRDS in always-on systems should weigh privacy, consent, and data-retention considerations alongside the efficiency benefits, and consider audit mechanisms when compressed video memory is used in safety- or rights-relevant decisions.

\end{document}